\documentclass[twoside,11pt]{article}

\usepackage{blindtext}

\usepackage{jmlr2e}

\usepackage{lastpage}
\jmlrheading{26}{2025}{1-\pageref{LastPage}}{10/25}{}{}{Donghwan Rho and Sieun Seo}
\usepackage[utf8]{inputenc} 
\usepackage[T1]{fontenc}    
\usepackage{hyperref}       
\usepackage{url}            
\usepackage{booktabs}       
\usepackage{amsfonts}       
\usepackage{nicefrac}       
\usepackage{microtype}      
\usepackage{xcolor}         
\usepackage{amsmath}
\usepackage{mathtools}
\usepackage{multirow}
\usepackage{float}
\usepackage{makecell}
\usepackage{array}
\usepackage{tabularx}
\usepackage{graphicx}
\definecolor{myPurple}{rgb}{0.5,0,0.5}
\usepackage{amssymb}
\usepackage{algpseudocode}
\usepackage{algorithm}
\usepackage{seqsplit}
\usepackage{enumitem}
\usepackage{contour}
\contourlength{0.8pt}
\usepackage{booktabs}
\usepackage{array}

\usepackage{longtable}
\usepackage{caption}
\usepackage{hyperref}

\newtheorem{assumption}[theorem]{Assumption}

\ShortHeadings{
TSP-Based Token Ordering in Homomorphically Encrypted Language Models
}{Rho, Seo, Sung, Min, and Ryu}
\firstpageno{1}

\begin{document}

\title{Traveling Salesman-Based Token Ordering Improves\\
Stability in Homomorphically Encrypted Language Models}

\author{\name Donghwan Rho\thanks{Equal contribution. Order determined by a coin toss.} \email donghwan\_rho@snu.ac.kr \\
       \addr Department of Mathematical Sciences\\
       Seoul National University\\
       Seoul, Korea
       \AND
       \name Sieun Seo\footnotemark[1] \email sieun1114@ewha.ac.kr \\
       \addr Department of Mathematics\\
       Ewha Womans University\\
       Seoul, Korea
       \AND
       \name Hyewon Sung \email hyewonsung@ewha.ac.kr \\
       \addr Department of Mathematics\\
       Ewha Womans University\\
       Seoul, Korea
       \AND
       \name Chohong Min\footnotemark[2] \email chohong@ewha.ac.kr \\
       \addr Department of Mathematics\\
       Ewha Womans University\\
       Seoul, Korea
       \AND
       \name Ernest K. Ryu\thanks{Co-corresponding authors.} \email eryu@math.ucla.edu \\
       \addr Department of Mathematics\\
       University of California, Los Angeles\\
       Los Angeles, United States of America
       }

\editor{XXX}

\maketitle

\begin{abstract}
As users increasingly interact with large language models (LLMs) using private information, secure and encrypted communication becomes essential. Homomorphic encryption (HE) provides a principled solution by enabling computation directly on encrypted data. Although prior work has explored aspects of running LLMs under HE, the challenge of text generation, particularly next-token prediction, has received limited attention and remains a key obstacle to practical encrypted interaction. In this work, we propose a TSP-based token reordering strategy to address the difficulties of encrypted text generation, together with a post-processing step that further reduces approximation error. Theoretical analysis and experimental results demonstrate that our method prevents collapse, improves coherence in generated text, and preserves data privacy throughout. Overall, our contributions advance the feasibility of practical and privacy-preserving LLM inference.
\end{abstract}

\section{Introduction}\label{sec:introduction}

Recent advances in transformer models, including the GPT family \citep{gpt-4,GPT-1,radford2019language,GPT-3} and the BERT family \citep{clark2020electra,he2020deberta,kenton2019bert,lan2019albert,liu2019roberta,sanh2019distilbert}, have transformed large language model (LLM) applications such as text generation, translation, and question answering. As these models are increasingly deployed across diverse domains, however, privacy concerns become critical, particularly when sensitive data is used to generate responses or predictions. This highlights the need for techniques that enable secure inference while preserving user data privacy. Fully Homomorphic Encryption (FHE) addresses this challenge by allowing computations on encrypted data, ensuring that neither the service provider nor external parties can access the underlying information. FHE supports encrypted computations with native operations for approximate addition, subtraction, and multiplication over both real and complex numbers. Its single-instruction-multiple-data (SIMD) capability further accelerates large-scale inner computations, making it especially well-suited for transformer inference. Since neural networks operate on real-valued (continuous) data, the CKKS (Cheon–Kim–Kim–Song) scheme \citep{cheon2017homomorphic,cheon2019full} emerges as a natural choice for applying homomorphic encryption to transformers and other language models.

However, implementing efficient transformer models under FHE poses challenges in next-token prediction. Decoding algorithms widely used in large language models such as argmax, top-$p$, and top-$k$ are difficult to realize in CKKS because they rely heavily on comparison operations, which cannot leverage SIMD and thus incur significant computational costs \citep{jovanovic2022private,zhang2024efficient,zhang2024secure}. Furthermore, after next-token prediction, the output index remains encrypted, making direct embedding look-up infeasible. As a workaround, the input to the subsequent layer is obtained by multiplying the embedding matrix by a one-hot vector. However, the intrinsic noise of FHE prevents constructing an exact one-hot vector, causing this matrix multiplication to yield a linear combination of semantically unrelated embedding vectors.

\paragraph{Contributions.}
In this work, we propose an efficient algorithm for next-token prediction under CKKS, along with a method that addresses the limitations of inexact one-hot vector representations. Our algorithm avoids max operations and sorting, relying solely on SIMD-friendly operations to enable practical homomorphic implementation. We show that the resulting errors are bounded by a small constant, ensuring controlled accuracy. However, even when errors are reduced to a negligible level, subsequent matrix multiplications may still yield semantically unrelated embedding combinations that affect model stability. To mitigate this, we adapt the traveling salesman problem (TSP) to reorder embedding vectors so that semantically similar vectors are placed adjacently, thereby preserving model stability in the presence of encryption noise.

\subsection{Prior work}

\paragraph{Language models under FHE.}
A variety of approaches have been proposed to address the main challenges: non-linear operations, matrix multiplication, and fine-tuning. For non-linear operations, \cite{lee2021minimax} approximates the sign function with minimax approximation, while \cite{lee2023precise} approximates ReLU and max-pooling in convolutional neural networks under FHE. \cite{rho2024encryption} replaces softmax with a Gaussian kernel to eliminate division and max operations, and \cite{cho2024fast} introduces a normalize-and-square method for accurate softmax approximation over a large range.
\cite{lee2023hetal} and \cite{moon2024thor} present algorithms for faster matrix multiplication. Finally, for fine-tuning, \cite{rho2024encryption} shows that LoRA reduces ciphertext–ciphertext multiplications. Beyond this line of work, we focus on random sampling for next-token prediction, contributing to the practical implementation of LLMs under FHE.

\paragraph{Next-token prediction methods of language models.}
Language models typically generate text via next-token prediction, and a variety of decoding algorithms have been explored both prior to and following the advent of language models. The most basic approaches include greedy decoding and probabilistic sampling \citep{graves2013generating}. To generate plausible text, more sophisticated strategies have been developed, such as beam search and threshold-based methods \citep{chen2023fidelity,meister2023locally,nguyen2024turning,vijayakumar2016diverse}. Among these, top-$k$ \citep{top-k} and top-$p$ sampling \citep{top-p} have become the most widely adopted methods, as they produce more natural text. However, they rely on conditional operations such as if-statements, which are inefficient to implement under CKKS.

\paragraph{Secure sampling methods under FHE.}
There are several papers for secure sampling. First, \cite{choi2022secure} address secure sampling in multi-party computation (MPC) \citep{yao1982protocols} settings under FHE, where parties such as server and client communicate for computation in encrypted states. However, MPC incurs substantial communication overhead, while our non-interactive setting avoids this cost. Besides these, there are a few works implementing sampling and argmax under FHE. For example, \cite{folkerts2024tyche} performs secure sampling under TFHE \citep{chillotti2020tfhe} and \cite{jovanovic2022private,zhang2024secure,zhang2024efficient} implement argmax under CKKS \citep{cheon2017homomorphic,cheon2019full}.
Concurrent to our work, \cite{avitan2025efficient} proposes CKKS-compatible argmax and nucleus sampling methods. While they focus on the next-token selection step itself, our work emphasizes the impact of embedding error accumulation and provides quantitative analysis of text quality.

\subsection{Preliminaries}\label{sec:preliminary}
\paragraph{Cheon--Kim--Kim--Song (CKKS).}
Fully homomorphic encryption (FHE) enables computation on encrypted data without decryption, thereby supporting privacy-preserving machine learning. Among various schemes, CKKS \citep{cheon2017homomorphic,cheon2019full} is particularly suitable for transformer inference because it supports efficient approximate arithmetic over real and complex numbers.

In CKKS, a ciphertext can encode a vector of complex numbers, where each entry of the vector is referred to as a slot. Basic arithmetic operations such as addition, multiplication, and rotation act in a SIMD (single instruction, multiple data) fashion, meaning that the same operation is applied to all slots in parallel. 
The supported operations are:
\vspace{3mm}
\begin{itemize}[leftmargin=*, itemsep=0pt, topsep=0pt]
    \item 
    \textbf{Addition}: Adds two ciphertexts, producing a ciphertext that encrypts the componentwise sum of their underlying slots.
    \item 
    \textbf{Multiplication}: Multiplies two ciphertexts, producing a ciphertext that encrypts the componentwise product of their underlying slots.
    \item \textbf{Rotation}: Applies a cyclic shift to the slots of a ciphertext, enabling the rearrangement of encrypted data and supporting more complex computations such as inner products and matrix multiplications.
\end{itemize}
\vspace{3mm}
These basic operations serve as the building blocks for more complex computations. In particular, any algorithm expressed through additions, multiplications, and rotations can be implemented homomorphically under CKKS. Based on these primitives, the following computations can be implemented:
\vspace{3mm}
\begin{itemize}[leftmargin=*, itemsep=0pt, topsep=0pt] 
    \item 
    \textbf{Matrix multiplication}: Matrix--vector multiplication can be implemented by combining the three basic operations. Multiplications compute slot-wise products, rotations align intermediate results to the correct positions, and additions accumulate them to form the final output. 
    \item 
    \textbf{Polynomial evaluation}: Polynomial evaluation is naturally supported since it only requires repeated additions and multiplications of ciphertexts. Owing to the SIMD structure, the same polynomial can be evaluated simultaneously across all slots.
\end{itemize}
\vspace{3mm}
While such computations are well aligned with the CKKS framework, many functions used in machine learning are non-polynomial in nature, which requires a different treatment. 
These functions are not directly supported in CKKS; instead, they are typically approximated by polynomials and then evaluated homomorphically using the procedure described above. 
In practice, approximation techniques such as Chebyshev series expansions or the Remez algorithm are often employed to construct low-degree polynomials with controlled error. For more details, see \cite{cheon2017homomorphic}.

\paragraph{Computation and threat model.}
We consider a server-client framework designed for private inference, operating under the standard honest-but-curious threat model. In this setting, the server is assumed to follow the computational protocol correctly, but may attempt to learn from all data it observes, such as encrypted inputs, intermediate results, and final outputs.

To ensure confidentiality, the client encrypts their input using a public FHE key and sends the ciphertext to the server.
The server, using its plaintext model weights, performs the inference homomorphically and returns the encrypted result. Only the client, who possesses the secret key, can decrypt this output. Owing to the semantic security of the underlying FHE scheme, the server learns nothing about the client's private data throughout the entire interaction.

\paragraph{Next token prediction of language models.}
We quickly describe the auto-regressive next-token prediction of decoder-only language models such as GPT \citep{gpt-4,GPT-1,GPT-2,GPT-3} and Llama series \citep{LLAMA3,LLAMA,llama-2}. Let $\mathcal{M}$ be a decoder-only language model with $L$ layers, $\mathcal{V}$ and $d$ denote the vocabulary and the embedding dimension of $\mathcal{M}$. Given a tokenized input ${X}=[x_0,\ldots,x_{t-1}]\in\mathbb{R}^{d\times t}$, if the output of the last hidden layer of $\mathcal{M}$ is $h_L\in\mathbb{R}^{d\times t}$, then the probability $P(x_{t}=v\,|\,x_0,\ldots,x_{t-1})$ of a token $v\in\mathcal{V}$ being selected as a next token $x_t$ is calculated as follows:
\begin{align*}
    Z_{t-1} &\coloneqq (z_0, \ldots, z_{|\mathcal{V}|-1}) = W h_L^{(t-1)}, \\
    P(x_t = v\,|\, x_0,\ldots,x_{t-1}) &\coloneqq \mathrm{Softmax}(Z_{t-1}) = \frac{\exp(z_v /T)}{\sum_{i=0}^{|\mathcal{V}|-1}\exp(z_i /T)}
\end{align*}
where $W\in\mathbb{R}^{|\mathcal{V}|\times d}$ is the embedding weight, $h_L^{(t-1)}$ is the last column of $h_L$, and $T>0$ is a temperature parameter. In practice, greedy decoding, top-$k$, and top-$p$ sampling are the most widely used methods. 
However, since these methods rely on many comparisons that are inefficient under CKKS, we focus on the probabilistic sampling: $x_t \sim P(\,\cdot\, |\,x_0,\ldots,x_{t-1})$.

\paragraph{Inverse transform sampling.}\label{sec:its}
Inverse transform sampling (ITS) is a standard technique that generates samples from a probability distribution using its cumulative distribution function (CDF) and uniform sampling. Let $X$ be a discrete real-valued random variable with support $\{x_0,\ldots, x_{|\mathcal{V}|-1}\}$ with $p_k=P(X=x_k)$ and define $s_k=\sum_{i=0}^{j} p_i$ $(k=0,\ldots,|\mathcal{V}|-1)$ with $s_{-1}=0$. The procedure of the ITS for $X$ is as follows: (i) Sample $U\sim\mathrm{Uniform([0,1])}$; (ii) find $k$ such that $s_{k-1}\le U < s_k$ (which is equivalent to $s_{k-1} < U \le s_k$, the definition of inverse transform sampling); (iii) return $X=x_k$. In this work, we adapt into a CKKS-friendly variant of the ITS, as described in Algorithm~\ref{alg:homomorphic_its}.

\section{Reordering tokens using TSP to prevent corrupted text generation}

In this section, we explain why the ordering of tokens influences generation under approximate computation, and propose to reorder the tokens using a traveling salesman problem (TSP) approach to prevent corrupted text generation.

\subsection{With imperfect sampling, adjacent tokens affect each other}

\begin{figure}[htbp]
    \centering
    \includegraphics[width=0.80\textwidth]{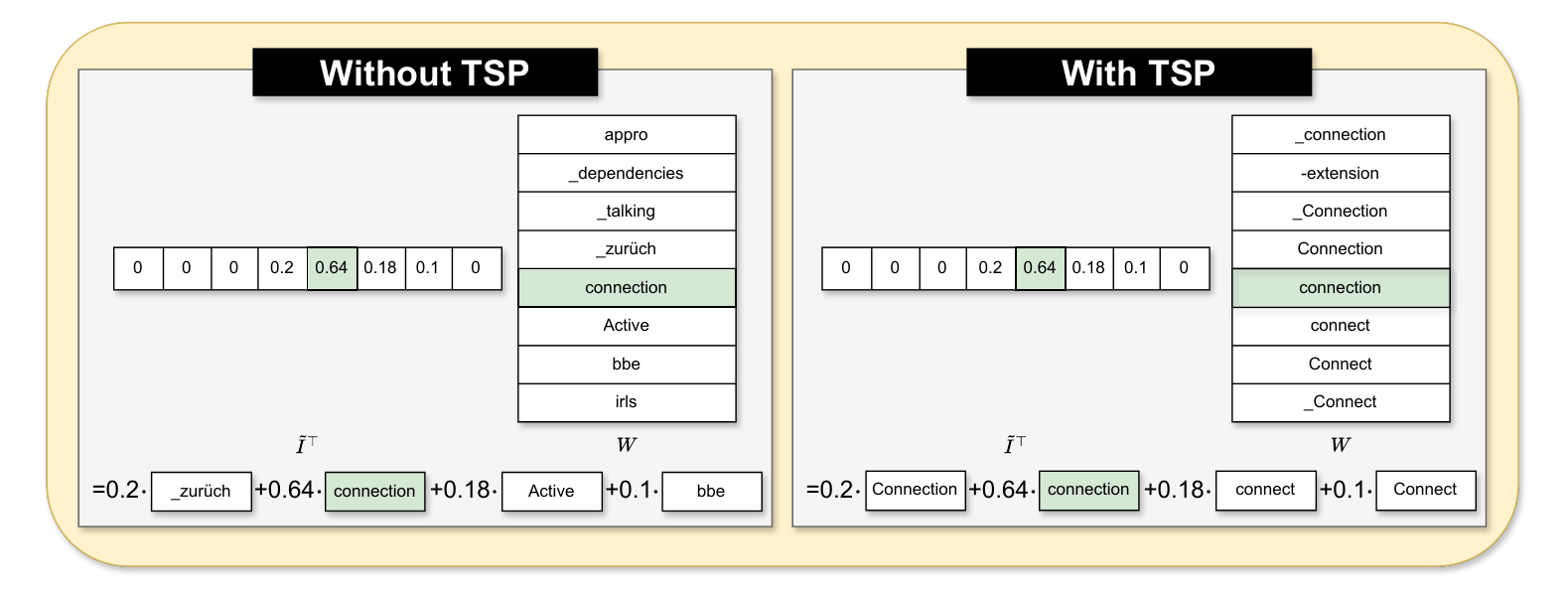}
    \caption{The effect of TSP. \textbf{(Left)} Without TSP, we obtain the linear combination of semantically irrelevant token embeddings. \textbf{(Right)} With TSP, semantically similar tokens can be combined.\label{fig:tsp}
    }
\end{figure}

In our random sampling algorithm, which will be described in Section~\ref{sec:homomorphic inverse transform sampling}, we aim to construct a one-hot vector $I \in \mathbb{R}^{|\mathcal{V}|}$ corresponding to the predicted next token (see Figure~\ref{fig:algorithm}). 
This design choice arises because, under CKKS, the predicted index remains encrypted and thus direct embedding lookup is not feasible. 
Instead, the input to the subsequent layer is obtained by multiplying the embedding matrix by the one-hot vector $I$. 
Formally, the embedding vector $w_i$ for the next token $i$ is given by $w_i = I^{\top} W$, where $W$ denotes the embedding weight matrix.

However, in the homomorphic setting, sampling cannot yield an exact one-hot vector. 
Due to unavoidable approximation errors, the resulting vector $\tilde{I}$ is imperfect and assigns small nonzero weights to multiple indices rather than a single one. 
To formalize this limitation, suppose that a probability vector $p$ and its cumulative vector $(s_k)_{k=0}^{|\mathcal{V}|-1}$ is given as in Section \ref{sec:preliminary}. Inverse transform sampling maps a uniform random sample $r \sim U[0,1]$ into a one-hot vector $\mathbf{e}_{i(r)}$ where $s_{i(r)-1}\le U< s_{i(r)}$. 
In the homomorphic setting, any algorithm can only produce an approximation $\tilde{I}$ of $\mathbf{e}_{i(r)}$, and the following theorem shows that the approximation error is necessarily nonzero.

\begin{theorem}\label{thm:nonzero_error}
Let $\mathbf{e}_{i(r)}$ denote the one-hot vector determined by inverse transform sampling 
from a random sample $r \sim U(0,1)$. 
For any homomorphic algorithm that computes an approximation $\tilde{I}$ of $\mathbf{e}_{i(r)}$, 
the following inequality holds:
\[
\mathbb{E}_{r}\!\left[\|\tilde{I} - \mathbf{e}_{i(r)}\|_\infty\right] > 0 ,
\]
where the expectation is taken over the random sample $r$.
\end{theorem}

\begin{proof}
For a random variable $r$, $i(r)$ is a step function. 
Any homomorphic evaluation can be represented by a polynomial function. 
Let the degree of this polynomial be $m$. 

According to the study on the uniform approximation of $\operatorname{sgn}(x)$~\citep{eremenko2006uniform}, the deviation between $\tilde{I}$ and $\mathbf{e}_{i(r)}$ satisfies
\[
\|\tilde{I} - \mathbf{e}_{i(r)}\|_\infty \;\ge\; 
\frac{1-\delta}{\sqrt{\pi \delta}}
\left(\frac{1-\delta}{1+\delta}\right)^m 
\frac{1}{\sqrt{m}}
\]
outside the $\delta$-neighborhood of its discontinuities, for any $\delta > 0$ and regardless of how large $m$ is. 
Consequently, we have
\[
\mathbb{E}_r\big[\|\tilde{I} - \mathbf{e}_{i(r)}\|_\infty\big] 
\;\ge\; (1-2\delta)\,\frac{1-\delta}{\sqrt{\pi \delta}}
\left(\frac{1-\delta}{1+\delta}\right)^m \frac{1}{\sqrt{m}}
\;>\; 0,
\]
which completes the proof.
\end{proof}

Consequently, the computed embedding $\tilde{I}^{\top} W$ is not a single token embedding $w_i$ but rather a linear combination of multiple embeddings, as illustrated in the left of Figure \ref{fig:tsp}.

In LLMs, however, tokens at adjacent indices generally have unrelated semantics (see the left of Table~\ref{tab:TSP_tokens}). 
As a result, when the approximate vector $\tilde{I}$ produces a linear combination of such embeddings, the outcome is semantically meaningless. 
Theorem~\ref{thm:nonzero_error} further shows that this issue cannot be avoided: even with highly accurate function approximation, an exact one-hot vector is unattainable, and thus approximation alone is not a fundamental solution. 
Therefore, we require a different approach. 
Several strategies can be considered to mitigate interference between embeddings of adjacent tokens, and we propose a TSP-based reordering method that places semantically similar tokens adjacently.

\begin{table}[!htbp]
\small
\begin{center}
\caption{\label{tab:TSP_tokens}Adjacent tokens of \texttt{Llama-2-7b-hf} before and after applying TSP ordering. The TSP places similar tokens in adjacent indices. The symbol `\texttt{\_}' represents a space.
}

\begin{tabular}{c@{\hskip 1pt}ccc@{\hskip 1pt}cc}
\toprule
          \multicolumn{3}{c}{Without TSP} & \multicolumn{3}{c}{With TSP} \\ \midrule
          \multicolumn{2}{c}{Index} & Token & \multicolumn{2}{c}{Index} & Token \\ \cmidrule(lr){1-3}\cmidrule(lr){4-6}
          \multicolumn{2}{c}{$\vdots$} & $\vdots$ & \multicolumn{2}{c}{$\vdots$} & $\vdots$ \\
          9961 &  & \texttt{appro} & 18927 & &\texttt{\_connection} \\
          9962 &  & \texttt{\_dependencies} & 18928 & & \texttt{\_extension} \\
          9963 &  & \texttt{\_talking} & 18929 & & \texttt{\_Connection} \\
          9964 &  & \texttt{\_zur\"uck} & 18930 & &  \texttt{Connection} \\
          9965 &  & \texttt{connection} & 18931 & & \texttt{connection} \\
          9966 &  & \texttt{Active} & 18932 & &  \texttt{connect} \\
          9967 &  & \texttt{bbe} & 18933 & &  \texttt{Connect} \\
          9968 &  & \texttt{irls} & 18934 & &  \texttt{\_Connect} \\
          9969 &  & \texttt{\_Inf} & 18935 & &  \texttt{\_connect} \\
          $\vdots$ & & $\vdots$ & $\vdots$ \\
          \bottomrule
\end{tabular}
\end{center}
\end{table}

\subsection{TSP order mitigates damage from imperfect sampling}\label{sec:TSP_reordering}

As a solution to the aforementioned problem, we propose applying the TSP to reorder the tokens to mitigate the damage from imperfect one-hot vectors.

When an imperfect one-hot vector is multiplied with the embedding
matrix, a weighted sum of multiple embedding vectors $\sum_{j=-m}^{n}c_{j}W_{i+j}$ is obtained in place of $W_i$ where $i$ is the selected next-token index with non-negative coefficients $c_{j}$.
However, as shown in the left of Table \ref{tab:TSP_tokens}, the surrounding tokens have semantically unrelated meanings, degrading the quality of the linear sum of obtained embedding vector. This erroneous embeddings accumulate during text generation, resulting in a collapsed text (see Table \ref{tab:Corrupted-Text}).

In contrast, if we rearrange the rows of the pretrained embedding matrix such that the embedding vectors of adjacent token indices are similar to each other,
then text generation collapse can be mitigated.
To resolve this, we minimize the sum of cosine distance of adjacent tokens:
\[
\underset{\pi}{\min}\sum_{i=1}^{|\mathcal{V}|-1}d_{\cos}(W_{\pi(i)},W_{\pi(i+1)}),
\]
where $\pi$ is a permutation and adopt a TSP as a solution.
By applying the TSP, we can place semantically similar tokens adjacently, therefore mitigating the problem mentioned in the previous section, as illustrated in the right of Figure \ref{fig:tsp}.
We visualize the effect of TSP token reordering in Figure \ref{fig:visualize_nic_ic}. In this figure, one can see that after reordering of tokens, the average cosine similarity between adjacent tokens substantially increases (from 0.024 to 0.271 in the case of \texttt{Llama-2-7b-hf}).

\begin{table}[t]
\begin{centering}
{\scriptsize{}}%
\begin{tabular}{cccccccccc}
\toprule 
{\scriptsize{}Prompt} & \multicolumn{9}{c}{{\scriptsize{}Please introduce yourself.}}\tabularnewline
\midrule 
{\scriptsize{}Answer} & \multicolumn{9}{c}{%
\begin{minipage}[t]{0.85\textwidth}%
{\scriptsize{}{}Please introduce yourself. Not long after finishing
the university in Moscow in H Physi am a friend at theTEenza stopped
escape to Moscow State University where our political large a large
number have metita I T a All Message When As Moscow USoni M MS To
was D MSNewisMSMSH MSMSA MSMSMSMS MS M. MM SWMMMS areMSMSIIMSMS M
MMT MMSMSMSZMSMSMS MSMS MMSMSMSMSMSMS MMSMSMSMSMS MS MMSMSMSMS MMSMSMSMSMS
SMS MMSMS WeMS MSMS MSMMSMMS I MOMS MS MS MJ MSMSMSMMS IMMSMSMSMSMSPA
MSMSMSMMSMSMSMSMSMS MSMS MSMS MSMMSItWACMSMS MSMSMMSMMS These IlMS
MSMSMSMSMSMSMSMSMSMSMSMSMSMS MMSMSMSMSMSMS MSMS M MMS MMSMSMSMSMSMSMSMSMSMSMSMS
MS MMSMS MSMSCCMS MMS MSMSMMSMSMS MSMS MMS \$MSMSMSMSNMSMSMSMSMSMSMSMSMSMAMMSMSMSMSMSMSMSMSMSMS
MSMSMSMMSMSMS MMSAAMS weMCMSMSMMSMSMSMSAAMSMMSMSMSMSMSMSMSMSMSMSMSMSMSMS
MSMSMSMSMS MANMSMSMSN Find MSMSMSMSMSMSMSMSMSMSMSMSMSNA MSMSMSMSMSMSMSMS
MSMMSMSMSMSMSMSMSMS STyMSMMS C IllPMSMSMSMSMSMSMSMSMSMSMSMS5 MSMS
GMSLAMSMS MSMSMSMSMS Project MSMSMSMSMSMSMS MSMS MSMSMS MSMSPMSMM
SKMSMSMSMSMSMSMSMMSMSMSMSMSMWMS MSMSMMSMSMSMSMSMSMSMSMSMSMSMSMSMS
MSMS MSMSM MarMSMSMSMSWMSMSMSMSMS MSMSMSMSMSMSMSMS MSMS MSMSMSPAMSMSMSMSMSMS
MSMS MSMSMSLMSMSMS MMSMSMSMS'MSMSMSMMSMMSMSlMSMSMMS MS MPA}%
\end{minipage}}\tabularnewline
\midrule 
{\scriptsize{}Perplexity} & \multicolumn{9}{c}{{\scriptsize{}12.6717}}\tabularnewline
\midrule 
{\scriptsize{}Score} & \multicolumn{9}{c}{{\scriptsize{}4}}\tabularnewline
\bottomrule
\end{tabular}{\scriptsize\par}
\par\end{centering}
\begin{centering}
\par\end{centering}
\caption{\label{tab:Corrupted-Text} An example of corrupted text generated by our homomorphic inverse transform sampling. The token `\texttt{MS}' is repeated meaninglessly.}
\end{table}

However, since the TSP is NP-hard \citep{Karp1972}, naively applying it to a large-size LLM vocabulary is prohibitively expensive. To circumvent this, we adopt a nearest neighbor heuristic \citep{TSP-heuristic}, which heuristically constructs an approximate solution with time complexity $\mathcal{O}(|\mathcal{V}|^2)$.

\section{CKKS-friendly and efficient SIMD-based inverse transform sampling}\label{sec:homomorphic inverse transform sampling}

\begin{figure}[t]
\begin{centering}
\includegraphics[width=0.80\textwidth]{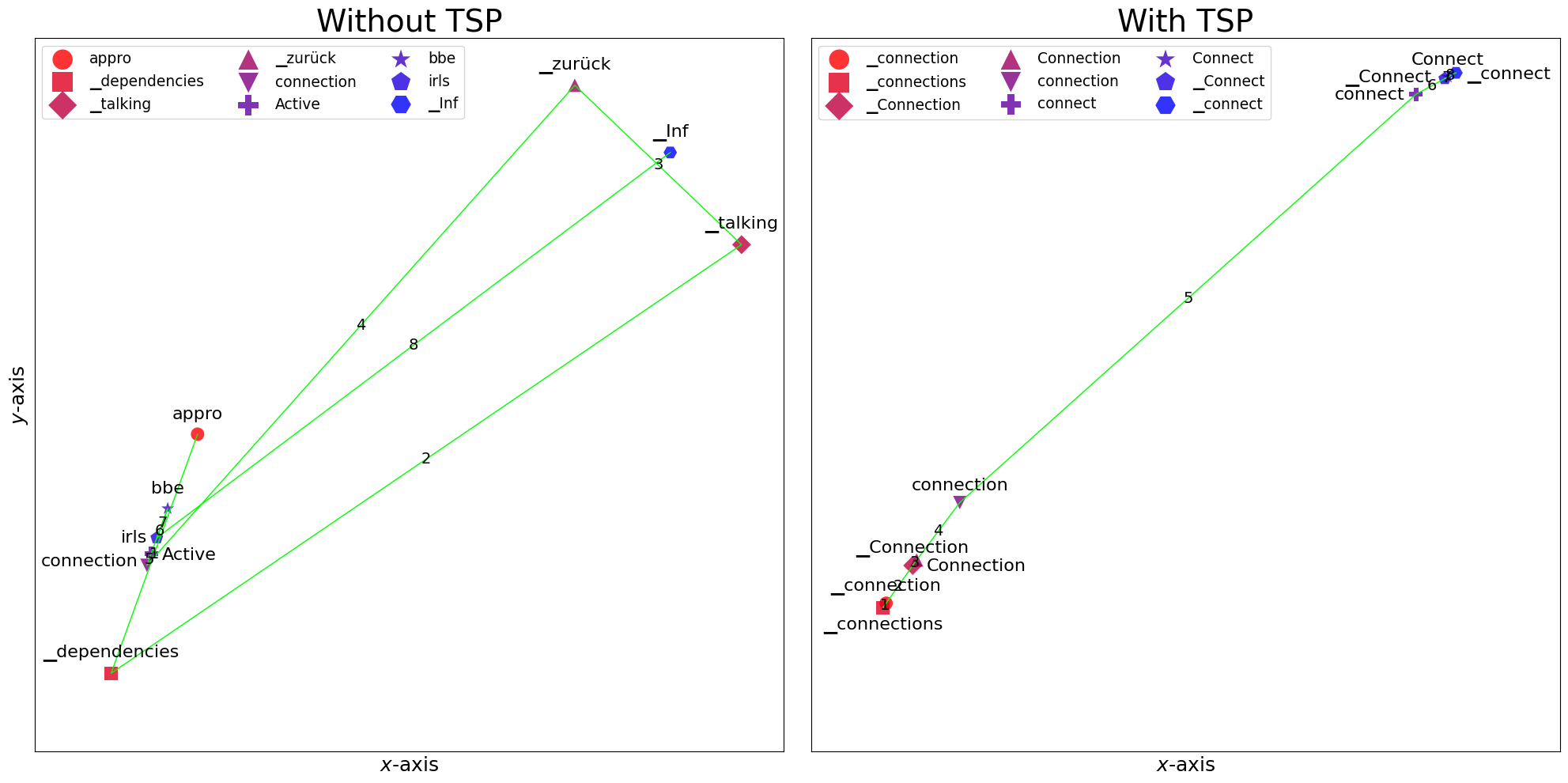}
\par\end{centering}
\caption{\label{fig:visualize_nic_ic}
Visualization of adjacent token embeddings in a 2D plane using UMAP \citep{umap}, with vectors normalized to the unit circle for cosine similarity comparison. In both figures, embedding vectors of tokens adjacent to the token `\texttt{connection}' (\textcolor{myPurple}{$\blacktriangledown$}) are plotted. Each green line connects two adjacent tokens, and longer lines indicate lower similarity. 
In the left figure, there are four long green lines, indicating the low cosine similarity between adjacent tokens. In contrast, with TSP (right), only one long line remains, showing that adjacent tokens are now semantically similar.
}
\end{figure}

In this section, we present an algorithm for next-token prediction suitable for FHE using inverse transform sampling.
Our algorithm starts from the softmax probability vector in next-token prediction stage, and the goal is to obtain the one-hot vector corresponding to the selected token.
As explained in the previous section, methods such
as argmax, top-$k$, and top-$p$ cannot utilize SIMD processing or divide-and-conquer technique (see Appendix \ref{sec:SIMD_FHE} and \ref{divide_and_conquer_FHE}), making them inefficient for FHE.
In contrast, our algorithm leverages efficient SIMD processing.
Our algorithm is described in Algorithm \ref{alg:homomorphic_its} and Figure \ref{fig:algorithm}.
$\widetilde{H}(x)$ in the algorithm refers to the approximation \eqref{eq:Htilde} of the Heaviside function $H(x)=\frac{1}{2}(\mathrm{sgn}(x)+1)$ where $\mathrm{sgn}$ is the sign function.

\begin{algorithm}[!htbp]
  \caption{CKKS-friendly inverse transform sampling}\label{alg:homomorphic_its}
  \begin{algorithmic}[1]
    \State \textbf{Input:} Probability vector $P$.
    \State \textbf{Output:} Approximate one-hot vector $\tilde{I}'$
    \Procedure{\,Homomorphic inverse transform sampling}{$P$}
        \State Sample $r \sim U(0,1)$.
        \State Compute the cumulative distribution vector $F$ of $P$.
        \State Apply the approximate Heaviside function: $h\leftarrow{\widetilde{H}(F-r)}$.
        \State Rotate the slots: ${\rm{Rot}}(h)$.
        \State{Calculate $\tilde{I}\leftarrow h\odot ( 1 - {\rm{Rot}}(h))$.}
        \State Apply post-processing: $\tilde{I}'\leftarrow{\rm{PP}}(\tilde{I}):=3\tilde{I}^2-2\tilde{I}^3$.
        \State \textbf{Return} \,$\tilde{I}'$
    \EndProcedure
  \end{algorithmic}
\end{algorithm}

\subsection{The sampling procedure}

\begin{figure}[t]
    \centering
    \includegraphics[width=0.98\textwidth]{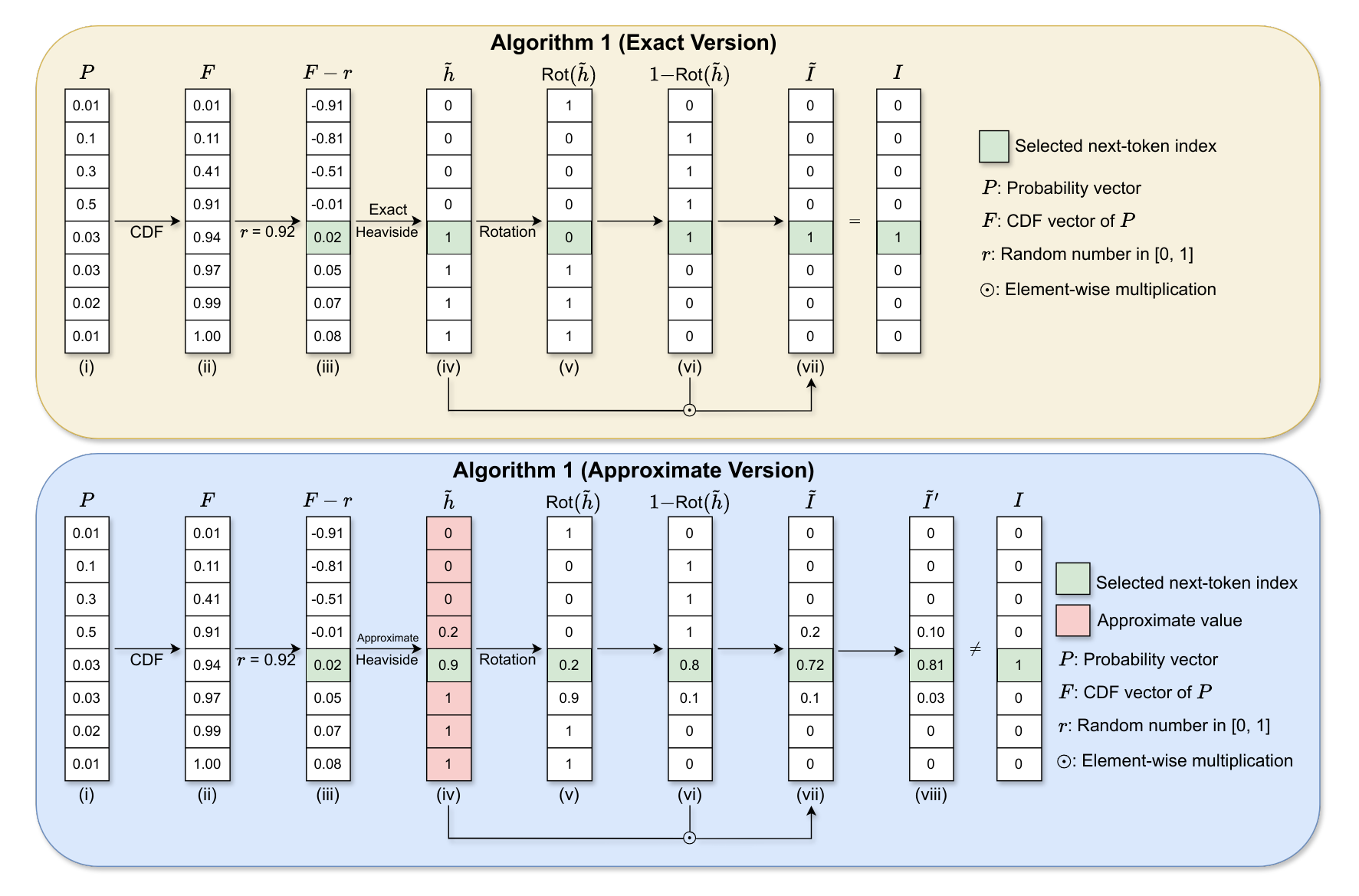}
    \caption{An illustrative example of Algorithm \ref{alg:homomorphic_its}. \textbf{(Top)} With the exact Heaviside function, we obtain the exact one-hot vector. \textbf{(Bottom)} With an approximate Heaviside function, the resulting vector is different from a one-hot vector.    Note that the steps are labeled (i)--(vii) for the exact version, with an additional step (viii) in the approximate version to enhance the one-hotness of $\tilde{I}$.
\label{fig:algorithm}
    }
\end{figure}

We explain Algorithm \ref{alg:homomorphic_its} and refer to Figure \ref{fig:algorithm} for intuitive understanding. 
In step (i) of Figure \ref{fig:algorithm}, a softmax probability vector $P=(p_i)\in\mathbb{R}^{|\mathcal{V}|}$ where $p_i$ denotes the probability of selecting token $i$ is given. In step (ii) we compute the cumulative distribution vector $F=(s_i)_{i=0}^{|\mathcal{V}|-1}$ as $s_i=\sum_{j=0}^i p_j$ with $s_{-1}=0$. 
In step (iii), sample $r\sim\mathrm{Uniform([0,1])}$ and compute $F-r$. In step (iv), apply the Heaviside function. The goal of this step is making each element of $F-r$ 0 or 1 (see the top of Figure \ref{fig:algorithm}). However, we have to approximate the non-polynomial Heaviside function, therefore the result of step (iv) is not 0 or 1 (see the bottom of Figure \ref{fig:algorithm}). In step (v)$\sim$(vii), we apply a rotation, subtraction, and element-wise multiplication. In step (viii), we apply a polynomial to enhance one-hotness of the resulting vector. Since we approximate the Heaviside function, $\tilde{I}$ is different from a one-hot vector, as shown in the bottom of Figure \ref{fig:algorithm}. Intuitively, we obtain the one-hot vector corresponding to the token at the first index where $F-r$ exceeds zero.

Step (ii) can be efficiently implemented as a single homomorphic matrix multiplication operation.
This is because $F$ can
be computed by multiplying $P$ with a lower triangular matrix filled with ones, as shown below:
\[
\left[\begin{array}{ccccc}
1 & 0 & 0 & \cdots & 0\\
1 & 1 & 0 & \cdots & 0\\
\vdots & \vdots & \vdots & \ddots & \vdots\\
1 & 1 & 1 & \cdots & 1
\end{array}\right]\left[\begin{array}{c}
p_{0}\\
p_{1}\\
\vdots\\
p_{|\mathcal{V}|-1}
\end{array}\right]=\left[\begin{array}{l}
p_{0}\\
p_{0}+p_{1}\\
\vdots\\
p_{0}+p_{1}+\cdots+p_{|\mathcal{V}|-1}
\end{array}\right].
\]
This operation involves a plaintext matrix multiplication with an encrypted vector. Given that the plaintext matrix is fully known, such plaintext-ciphertext operations can be performed with high efficiency under homomorphic encryption, particularly when leveraging optimized linear algebra libraries, as described in \cite{bae2024plaintext}.
Step (iv) and (viii) is implemented efficiently, utilizing SIME processing.
Also, no homomorphic operation is required
to sample $r$.
And we need only one homomorphic rotation in step (v) and subtraction and element-wise multiplication of vectors are also cheap.

\begin{table}[!htbp]
\scriptsize
\begin{center}
\caption{\label{tab:weighted_sum_can_exceed_1}
The effects of post-processing.
\textbf{(Case 1)} An example of a weight vector obtained using our algorithm where the sum of all weights exceedingly 1. After applying post-processing, the sum of all weights becomes closer to 1.
\textbf{(Case 2)} An example of a weight whose sum is near 1 but the dominant value is far less than 1. After post-processing, the sum of all weights is still approximately 1, but the dominant weight increases, and the other weights significantly decrease.
}
\begin{tabular}{ccccccccccc}
\toprule
          & & \multicolumn{8}{c}{Top $i$-th Element} & \\ \cmidrule(lr){3-10}
          & Post-processing & 1 & 2 & 3 & 4 & 5 & 6 & $\cdots$ & 512 & Sum \\\midrule
          \multirow{2}{*}{Case 1} & False & 0.4575 & 0.4335 & 0.1816 & 0.0942 & 0.0863 & 0.0863 & $\cdots$ & 0.0049 & 3.9572 \\
          & True & 0.4364 & 0.4008 & 0.0870 & 0.0249 & 0.0211 & 0.0211 & $\cdots$ & $7.18e-5$ & \textbf{1.0504}\\\midrule
          \multirow{2}{*}{Case 2} & False & 0.9292 & 0.0413 & 0.0227 & 0.0198 & 0.0003 & 0.0000 & $\cdots$ & 0.0000 & 1.0133 \\
          & True & \textbf{0.9857} & 0.0050 & 0.0015 & 0.0012 & 0.0000 & 0.0000 & $\cdots$ & 0.0000 & \textbf{0.9934}\\
          
          \bottomrule
\end{tabular}
\end{center}
\end{table}

Finally, we explain step (viii) in the bottom of Figure \ref{fig:algorithm}, which we call \textit{post-processing}. As stated in the previous section, although we approximate the Heaviside function with high accuracy, sometimes the resulting vector can be significantly different from a one-hot vector as shown in Table \ref{tab:weighted_sum_can_exceed_1}.
For example, the sum of the resulting vector can significantly exceed 1. Further, although the sum of the resulting vector is approximately 1, the dominant value can be much less than 1. These cases are problematic since they result in `big' or `mixed' embedding vectors. To remedy these, we propose a function PP defined as $\text{PP}(x)=-2x^3+3x^2$, which is also used in \cite{HELUT}. This function makes an element close to 0 closer to 0 and an element close to 1 closer to 1. It is FHE-friendly since the degree of $f$ is low, consuming low depth. Post-processing is applied in step 9 of Algorithm \ref{alg:homomorphic_its}. By applying post-processing, the sum of the resulting vector of our algorithm becomes closer to a one-hot vector as in Table \ref{tab:weighted_sum_can_exceed_1}. See Appendix \ref{sec:post_processing_necessity} for the case analysis of the need for post-processing. Section \ref{sec: experimental results} presents the results of the ablation study that evaluate the effectiveness of post-processing in improving text generation quality.

\subsection{Imperfect one-hot vector generation in homomorphic inverse transform
sampling\label{subsec:Ambiguous-One-Hot-Vector}}

The problematic part of Algorithm \ref{alg:homomorphic_its} is step (iv). If we apply the Heaviside function in step (iv), it is straightforward to see that a one-hot vector can be obtained in step (vii). However, since we have to approximate the discontinuous Heaviside function, in step (iv), the resulting vector $\tilde{I}$ obtained in step (vii) is different from one-hot. This error implies that several tokens may be chosen simultaneously. The error from the Heaviside is particularly considerable when the input is near zero. Also, this error compounds during text generation as explained in Section \ref{sec:TSP_reordering}. This motivates the use of TSP-based token reordering described in Section \ref{sec:TSP_reordering}.

\subsection{Theoretical analysis of our sampling method}
In this section, we establish error bounds for the approximate one-hot vectors and embedding representations, and analyze the effects of TSP token reordering and post-processing. To this end, we introduce the notion of a \textit{good event} and impose two natural assumptions, motivated by the approximation of the Heaviside function and \cite{top-p}, respectively.

\begin{assumption}\label{assumption:approx_heaviside}
The approximation $\widetilde{H}$ of the Heaviside function satisfies the following properties:
    \begin{enumerate}[label=\alph*)]
    \item $0\le \widetilde{H}(x)\le 1$ and $\widetilde{H}(x)=1-\widetilde{H}(1-x)$ $\forall x\in[-1, 1]$.
    \item $\text{Given }\varepsilon>0,~\exists\,\delta>0\text{ s.t. }x\in[-1,-\delta]\Rightarrow \widetilde{H}(x)\in[0,\varepsilon]\text{ and } x\in[\delta,1]\Rightarrow \widetilde{H}(x)\in[1-\varepsilon,1]$.
\end{enumerate}
\end{assumption}

The parameters $\varepsilon$ and $\delta$ capture the accuracy–margin trade-off of $\widetilde{H}$ and both of them are very small as in figure \ref{fig:heaviside-delta}. In what follows, we fix $\varepsilon$ together with its corresponding $\delta$.

\begin{definition}
    Given a probability vector $p$ and its cumulative probability $(s_i)_{i=0}^{|\mathcal{V}|-1}$ where $s_{-1}=0$ and $s_i=\sum_{j=0}^i p_j$, we define a \textit{good event} $\mathcal{G}_p$ as follows:
\[
    \mathcal{G}_p=\{r\in[0,1]\,|\,\exists\, k\text{ s.t. }s_{k-1}+\delta \le r < s_k-\delta\}
\]
and a bad event as $\mathcal{G}_p^c$.
\end{definition}
Finally, following \cite{top-p}, we assume that a small number of tokens account for the majority of the probability mass in next-token prediction.

\begin{assumption}\label{assumption:top-p}
The output distribution of the LLM is peaked, meaning the number of tokens that capture most of the probability mass is much smaller than the total vocabulary size. That is, $k_{\text{eff}} \ll |\mathcal{V}|$ tokens together capture a probability mass of at least $1-\varepsilon_{\text{tail}}$, for some small $\varepsilon_{\text{tail}} > 0$.
\end{assumption}

We are now ready to state our theoretical results, which provide upper bounds on the expected errors of our methods.

\begin{theorem}\label{thm:appr_error}
Let $P$ be a given probability vector and let $r\sim\mathrm{Uniform}([0,1])$. 
For the outputs $\tilde{I}$ and $\tilde{I}'$ of Algorithm \ref{alg:homomorphic_its}, their deviations from the one-hot vector $\mathbf{e}_{i(r)}$ satisfy the following inequalities:
\begin{align*}
\mathbb{E}_{r}\big[\|\tilde{I}-\mathbf{e}_{i(r)}\|_\infty\big] &\ \le \ 2\varepsilon + 2k_{\mathrm{eff}}\delta + \varepsilon_{\mathrm{tail}}, \\ \mathbb{E}_{r}\big[\|\tilde{I}'-\mathbf{e}_{i(r)}\|_\infty\big] &\ \le \ 12\varepsilon^2 + 2k_{\mathrm{eff}}\delta + \varepsilon_{\mathrm{tail}}.
\end{align*}
\end{theorem}

Since $12\varepsilon^{2}$ is much smaller than $2\varepsilon$, 
the post-processing step effectively reduces the expected error.
\vspace{3mm}

\begin{proof}  
According to Assumption \ref{assumption:approx_heaviside}, $\widetilde{H}(x)$ satisfies the error bound $\varepsilon$ outside the $\delta$-neighborhood of its discontinuity. Let $\mathcal V$ be the vocabulary with distribution $p=(p_j)_{j=0}^{|\mathcal V|-1}$, and define the cumulative sums $s_i=\sum_{j=0}^{i} p_j$ for $i=0,\dots,|\mathcal V|-1$ and $s_{-1}=0$. The inverse-transform output $\mathbf{e}_{i(r)}$ is discontinuous at each threshold $r=s_i$.
For $j=i(r)$,
\begin{align*}
|\tilde{I}_j - 1 | & = |\widetilde{H}(s_j-r)(1-\widetilde{H}(s_{j-1}-r)-1| = |\widetilde{H}(s_j -r)\widetilde{H}(r-s_{j-1})-1| \\
& = 1-\widetilde{H}(s_j -r)\widetilde{H}(r-s_{j-1}) \le 1-(1-\varepsilon)^2=2\varepsilon-\varepsilon^2\\
& <2\varepsilon.
\end{align*}
Also, if $j \ne i(r)$, then $s_{j-1}-r\ge \delta$ or $s_j-r\le -\delta$. Therefore,
\begin{align*}
|\tilde{I}_j - 0 | & = \widetilde{H}(s_j -r)\widetilde{H}(r-s_{j-1})\le 1\cdot \varepsilon=\varepsilon.
\end{align*}
Hence we have $\Vert \tilde{I}-\mathbf{e}_{i(r)}\Vert_{\infty}\le 2\varepsilon$.

\noindent From Assumption \ref{assumption:top-p}, the next-token distribution $p$ is concentrated on a small head of size $k_{\mathrm{eff}}$ that carries at least $1-\varepsilon_{\mathrm{tail}}$ of the total mass. Let $H$ denote the set of indices of the top-$k_{\mathrm{eff}}$ tokens by probability (the \emph{Head}), and let $T := H^{\mathrm c}$ denote the remaining indices (the \emph{Tail}).  
Decomposing $\mathcal{G}_p^c$ into contributions from Head and Tail gives

\begin{align*}
\mathbb{P}[\mathcal{G}_{p}^{c}] & = \mathbb{P}\left[r\in\bigcup_{i=0}^{\left|V\right|-1}\left(\left[s_{i-1},s_{i}\right]-\left[s_{i-1}+\delta,s_{i}-\delta\right)\right)\right]\\
& = \mathbb{P}\Big[r\in\bigcup_{i\in H}([s_{i-1},s_{i}]-[s_{i-1}+\delta,s_{i}-\delta])\Big]+\mathbb{P}\Big[r\in\bigcup_{i\in T}([s_{i-1},s_{i}]-[s_{i-1}+\delta,s_{i}-\delta])\Big]\\
 & \leq \sum_{i\in H}2\delta+\sum_{i\in T}p_{i}=
  2k_{\text{eff}}\delta+\varepsilon_{\text{tail}}.
\end{align*}

Combining the fact that $\|\tilde{I}-\mathbf{e}_{i(r)}\|_\infty \le 1$ under $\mathcal{G}_p^c$ with the previous results, we obtain the following upper bound on the approximation error of $\tilde{I}$:

\begin{align*}
\mathbb{E}_r[\Vert \tilde{I}-\mathbf{e}_{i(r)}\Vert _{\infty}] &\ \le \ \mathbb{P}[\mathcal{G}_{p}]\cdot2\varepsilon + \mathbb{P}[\mathcal{G}_{p}^{c}]\cdot1 \ \le \ 2\varepsilon+2k_{\text{eff}}\delta+\varepsilon_{\text{tail}}.
\end{align*}

On the interval $[0,1]$, the post-processing function $2x^3 - 3x^2$ satisfies the following inequalities:

\begin{align*}
\left|(2x^{3}-3x^{2})-0\right| \leq 3x^{2},\qquad
\left|(2x^{3}-3x^{2})-1\right| \leq 3(x-1)^{2}.
\end{align*}
Therefore, for the post-processed $\tilde{I}'$ if $j=i(r)$, then
\[
    |\tilde{I}'_j - 1| = |\text{PP}(\tilde{I}_j)-1|\le 3(\tilde{I}_j -1)^2\le 12\varepsilon^2.
\]
If $j\ne i(r)$, then
\[
    |\tilde{I}'_j - 0| = |\text{PP}(\tilde{I}_j)|\le 3\tilde{I}_j^2 \le 3\varepsilon^2.
\]
Hence, the upper bound for $\Vert\tilde{I}'-\mathbf{e}_{i(r)}\Vert_{\infty}$ can be derived in exactly the same manner as for $\Vert\tilde{I}-\mathbf{e}_{i(r)}\Vert_{\infty}$, and is therefore omitted.
\end{proof}

Finally, we establish a theorem about the effect of TSP token reordering.

\begin{theorem}\label{thm:TSP_not_pp}
    \textbf{(Effect of TSP with compact support)} Suppose that a probability vector $p$ is given, $r\sim\mathrm{Uniform}([0,1])$, and its corresponding index is $i(r)$.
Let $\tilde I$ and $\tilde{I}'$ be the outputs of Algorithm~\ref{alg:homomorphic_its}, and assume there exists $R\in\mathbb N$ such that $\tilde I_j=0$ and $\tilde I'_j=0$ for $\lvert j-i(r)\rvert>R$.
For the embedding matrix $W\in\mathbb R^{\lvert\mathcal V\rvert\times d}$ and its normalized row vectors $\bar W_i:=W_i/\|W_i\|_2$, define
\[
\kappa_R \coloneqq \sup_{k} \max_{\lvert t\rvert\le R} \frac{\|W_{k+t}\|_2}{\|W_{k}\|_2},\ \bar d_t^{(p)}:=\sum_{i=0}^{\lvert\mathcal V\rvert-1}p_i\,d_{\cos}(\bar W_{i+t},\bar W_{i+t+1}), \text{ and } D_{R}^{(p)}\coloneqq\sum_{j=1}^{R}j\sum_{t=-j}^{\,j-1}\bar{d}_t^{(p)}.
\]
Then for $p_{b}:=\mathbb P(\mathcal G_p^c)$ and the target embedding vector $\bar{W}_{i(r)}$, the followings hold:

\begin{align*}
    \textit{If $(1-\varepsilon)^2\ge 2R\varepsilon\kappa_R$, } & \mathbb{E}_r \big[d_{\cos}(\tilde{I}^{\top}W, \bar{W}_{i(r)})\big] \le
(1-p_{\text{b}})\frac{\varepsilon\kappa_R}{(1-\varepsilon)^2}D_{R}^{(p)}
 + 2p_{\text{b}}. \\
\textit{If $1-3(2\varepsilon-\varepsilon^2)^2 \ge 6R\,\varepsilon^2\,\kappa_R$, } & \mathbb{E}_r \big[d_{\cos}(\tilde{I}'^{\top}W, \bar{W}_{i(r)})\big] \le
(1-p_{\text{b}})\frac{3\,\varepsilon^2\,\kappa_R}{\,1-3(2\varepsilon-\varepsilon^2)^2\,}D_{R}^{(p)}
 + 2p_{\text{b}}.
\end{align*}
\end{theorem}

Intuitively, $\bar{d}_t^{(p)}$ is the weighted sum of the cosine distances between the adjacent embedding vectors. After applying TSP, the upper bound for the expected distance between the obtained and the target embedding vectors decreases.
\vspace{3mm}

\begin{proof}
We only prove the formal inequality; the proof for for the latter is similar. Decompose $\mathbb{E}_r \big[d_{\cos}(\tilde{I}^{\top}W, \bar{W}_{i(r)})\big]$ as
\[
\mathbb{E}_r \big[d_{\cos}(\tilde{I}^{\top}W, \bar{W}_{i(r)})\big]
= (1-p_{b})\,\mathbb{E}_r \big[d_{\cos}(\tilde{I}^{\top}W, \bar{W}_{i(r)})\,\big|\, \mathcal{G}_p\big]
+ p_{b}\,\mathbb{E}_r \big[d_{\cos}(\tilde{I}^{\top}W, \bar{W}_{i(r)})\,\big|\, \mathcal{G}_p^c\big].
\]
Since $d_{\cos}\in[0,2]$, the second term is bounded by $2p_{b}$. We work on $\mathcal G_p$.

\noindent For $r\in\mathcal G_p$, Assumption~1 on $\widetilde{H}$ gives
\begin{equation}\label{eq:I_bounds_patch2}
\tilde I_{i(r)}\ \ge\ (1-\varepsilon)^2,\qquad
\tilde I_j\ \le\ \varepsilon\ \ (j\ne i(r)),\qquad
\tilde I_j=0\ \ (\lvert j-i(r)\rvert>R).
\end{equation}

\noindent Rewrite $
\tilde{I}^{\top}W\ =\ \sum_{j}\tilde I_j W_j\ =\ \sum_{j}\tilde I_j\|W_j\|_2\,\bar W_j.
$
Let $S:=\sum_{j}\tilde I_j\|W_j\|_2$ and $w_j:=\tilde I_j\|W_j\|_2/S$. Then
\[
\sum_j w_j=1,\qquad w_j\ge0,\qquad  \text{and}\qquad  \tilde{I}^{\top}W\ = Sm
\]
where $m=\sum_{j} w_j\,\bar W_j$.
Therefore, $d_{\cos}(\tilde{I}^{\top}W,\bar W_{i(r)})=d_{\cos}(m,\bar W_{i(r)})$ by scale-invariance. Also $\|m\|_2\le 1$.
By \eqref{eq:I_bounds_patch2} and the definition of $\kappa_R$,
\[
S\ \ge\ \tilde I_{i(r)}\|W_{i(r)}\|_2\ \ge\ (1-\varepsilon)^2\|W_{i(r)}\|_2,\qquad
\|W_{i(r)\pm s}\|_2\ \le\ \kappa_R\,\|W_{i(r)}\|_2\ \ (1\le s\le R).
\]
Hence, for $1\le s\le R$,
\begin{equation}\label{eq:w_offk_patch2}
w_{i(r)\pm s}\ =\ \frac{\tilde I_{i(r)\pm s}\|W_{i(r)\pm s}\|_2}{S}\ \le\ \frac{\varepsilon\,\kappa_R}{(1-\varepsilon)^2},
\end{equation}
and by $\displaystyle S= \sum_{j=-R}^{R}\tilde{I}_j \|W_{i(r)}\|^2 \le  \tilde{I}_{i(r)}\| W_{i(r)}\|_2 + \sum_{0<|j|\le R}\!\!\!\varepsilon\kappa_R \|W_{i(r)}\|_2 = (\tilde{I}_{i(r)}+2R\varepsilon\kappa_R )\|W_{i(r)} \|_2$, we have
\begin{equation}\label{eq:wk_lower_patch2}
w_{i(r)}\ =\ \frac{\tilde I_{i(r)}\|W_{i(r)}\|_2}{S}\ \ge\ \frac{(1-\varepsilon)^2}{(1-\varepsilon)^2+2R\,\varepsilon\,\kappa_R}.
\end{equation}

\noindent The assumed condition $(1-\varepsilon)^2\ge 2R\,\varepsilon\,\kappa_R$ implies $w_{i(r)}\ge\tfrac12$ by \eqref{eq:wk_lower_patch2}.
Therefore $\langle m,\bar W_{i(r)}\rangle\ge w_{i(r)}-\sum_{i\ne k}w_i=2w_{i(r)}-1\ge 0$, and with $\|m\|_2\le 1$,
\begin{equation}\label{eq:mixture_patch2}
d_{\cos}(m,\bar W_{i(r)})\ =\ 1-\frac{\langle m,\bar W_{i(r)}\rangle}{\|m\|_2}
\ \le\ 1-\langle m,\bar W_{i(r)}\rangle
\ =\ \sum_{j} w_j\,d_{\cos}(\bar W_j,\bar W_{i(r)}).
\end{equation}
As $d_{\cos}(\bar W_{i(r)},\bar W_{i(r)})=0$ and $w_j=0$ for $\lvert j-i(r)\rvert>R$, \eqref{eq:w_offk_patch2} yields
\[
d_{\cos}(m,\bar W_{i(r)})\ \le\ \frac{\varepsilon\,\kappa_R}{(1-\varepsilon)^2}\sum_{s=1}^{R}\Big\{d_{\cos}(\bar W_{i(r)+s},\bar W_{i(r)})+d_{\cos}(\bar W_{i(r)-s},\bar W_{i(r)})\Big\}.
\]

\noindent For unit vectors $x_0,\dots,x_j$, the chord-triangle inequality gives
\[
d_{\cos}(x_j,x_0)\ \le\ j\sum_{t=0}^{j-1} d_{\cos}(x_{t+1},x_t).
\]
Applying this in both directions,
\[
d_{\cos}(\bar W_{i(r)\pm s},\bar W_{i(r)})\ \le\ s\sum_{t=0}^{s-1} d_{\cos}(\bar W_{i(r)\pm t\pm 1},\bar W_{i(r)\pm t}).
\]
Therefore, for fixed $r\in\mathcal G_p$,
\[
d_{\cos}(\tilde{I}^{\top}W,\bar W_{i(r)})\ \le\ \frac{\varepsilon\,\kappa_R}{(1-\varepsilon)^2}\sum_{s=1}^{R} s\sum_{t=0}^{s-1}\Big\{d_{\cos}(\bar W_{i(r)+t+1},\bar W_{i(r)+t})+d_{\cos}(\bar W_{i(r)-t-1},\bar W_{i(r)-t})\Big\}.
\]

\noindent On the good event, taking expectation over $r$ (so $k$ is distributed according to $p$), we get
\[
\mathbb E_r\!\big[d_{\cos}(\bar W_{i(r)+t+1},\bar W_{i(r)+t})\mid \mathcal G_p\big]
=\sum_{i}p_i\,d_{\cos}(\bar W_{i+t+1},\bar W_{i+t})
=\bar d_t^{(p)},
\]
and similarly for $-t-1$. Hence we have the following:
\[
\mathbb E_r\!\big[d_{\cos}(\tilde{I}^{\top}W,\bar W_{i(r)})\mid \mathcal G_p\big]
\ \le\ \frac{\varepsilon\,\kappa_R}{(1-\varepsilon)^2}\sum_{s=1}^{R} s\sum_{t=0}^{s-1}\Big(\bar d_t^{(p)}+\bar d_{-t-1}^{(p)}\Big)
\ =\ \frac{\varepsilon\,\kappa_R}{(1-\varepsilon)^2}\sum_{s=1}^{R} s\sum_{t=-s}^{\,s-1}\bar d_t^{(p)},
\]
which completes the proof.
\end{proof}

\newpage 

\section{Experiments}\label{sec: experimental results}

In this section, we explain our HE-based text generation scheme and experimental results.

\begin{center}
\begin{figure}[!htbp]
    \centering
    \includegraphics[width=0.95\textwidth]{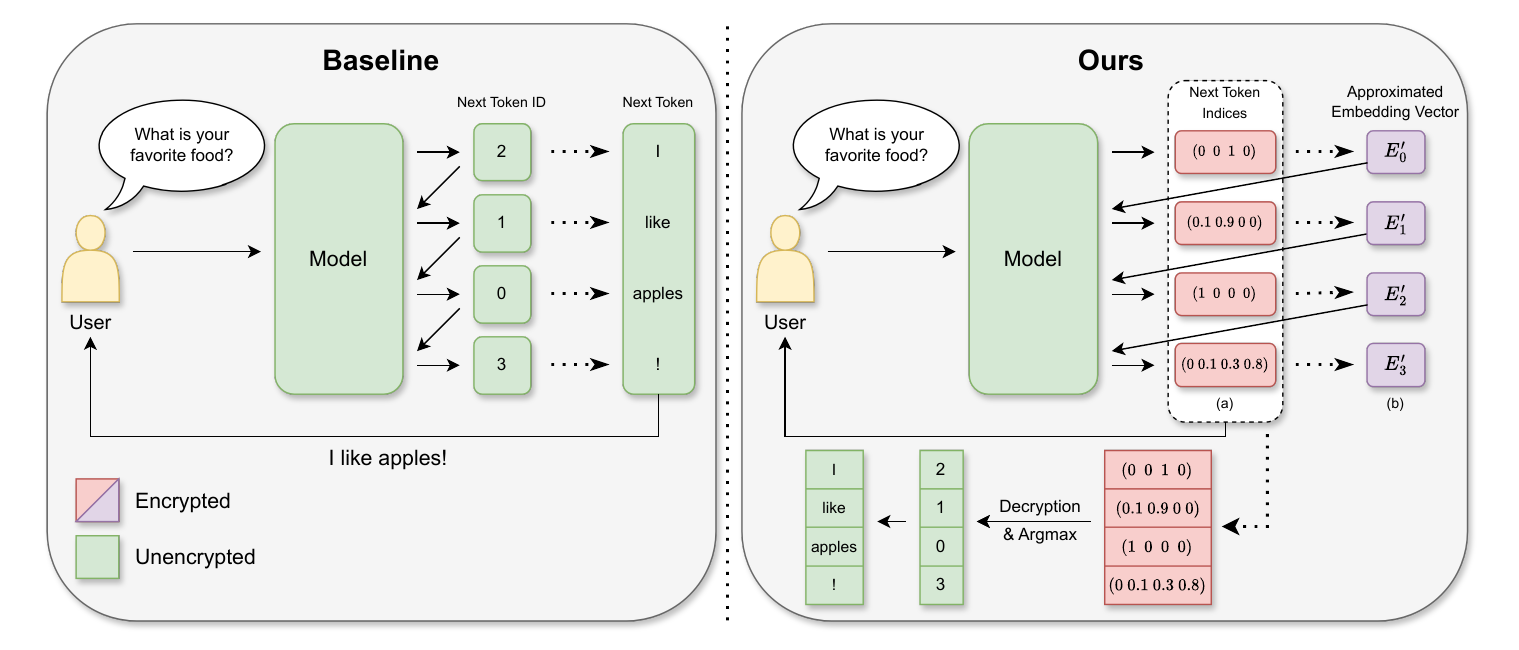}
    \caption{
    The schematic illustration of our work. \textbf{(Left)} Standard next token prediction. \textbf{(Right)} Our HE-based text generation. First, a user encrypts the embedding vectors of the prompt and the model performs the inference. However, since the output is a ciphertext, the model cannot select the index of the next-token. Instead, (a) the model approximates the one-hot vector of the next token (red), and (b) computes the approximated embedding vector (purple) and concatenates it into the model input. After generation, the model sends the concatenation of (a) to the user, who decrypts and applies argmax to recover to token. See Figure \ref{fig:tsp} and Section \ref{sec:TSP_reordering} for details on how (a) and (b) affect text generation.
    }
    \label{fig:scheme}
\end{figure}
\end{center}

\subsection{HE-based text generation}
We now explain text sampling under HE, illustrated in Figure~\ref{fig:scheme}. In standard next token prediction, a model predicts the indices of the next tokens. The user receives and decodes them into tokens. In contrast, under HE, we cannot use operations such as $\max$ or an if-statement to predict a particular token index. Therefore a model saves weighted indices in next token predictions and sends them to the user. Then the user decrypts the encrypted weighted indices and apply $\mathrm{argmax}$ decode to get the text.

We conduct our experiments under plaintext, not under HE, due to the limitations on computational budget. However, we anticipate that the result would not be significantly different under HE as can be found in the prior work \citep{lee2023hetal, rho2024encryption}.

\subsection{Criteria for text evaluation}
We experimentally show that our TSP-based token reordering helps an LLM generate higher-quality texts, and post-processing and domain-specific fine-tuning (refer to Section \ref{sec:domain_specific_fine_tuning}) further provide auxiliary improvements in generation quality. In our experiments, we use \texttt{Llama2-7b-hf} \citep{llama-2} to generate texts. First, we define what is a corrupted text and two metrics for evaluating the ratio of the corrupted text and how a generated text is coherent. In experiments, we observed that when the generation by the model breaks, the model generates the token `\texttt{MS}' repeatedly and meaninglessly and we say the text is \textit{corrupted}. The \textit{corruption ratio} is defined as the percentage of the corrupted texts among all generated texts. The \textit{corruption score} measures how many parts of a generated text is incoherent. We let GPT-4 \citep{gpt-4} API (\texttt{gpt-4-0613}) grade the score on each text according to the degree of the corruptness of the text. If 0$\sim$20\% of a text is incoherent (e.g., hard to understand or inconsistent), then the text is assigned a corruption score of 0. On the other hand, if 80$\sim$100\% of a text is incoherent, then its corruption score is 4. We use these scores as the metrics to evaluate the quality of generated texts. Exceptionally, if the model generates repeated `\texttt{MS}' tokens, we automatically give this text a score of 4. The prompt to grade the score can be seen in Table \ref{table:quality criteria}. Appendix \ref{sec:generated_results} presents examples of well and poorly written texts with the corruption score 0 and 4.

\begin{center}
\begin{figure}[!htbp]
    \centering
    \includegraphics[width=0.9\textwidth]{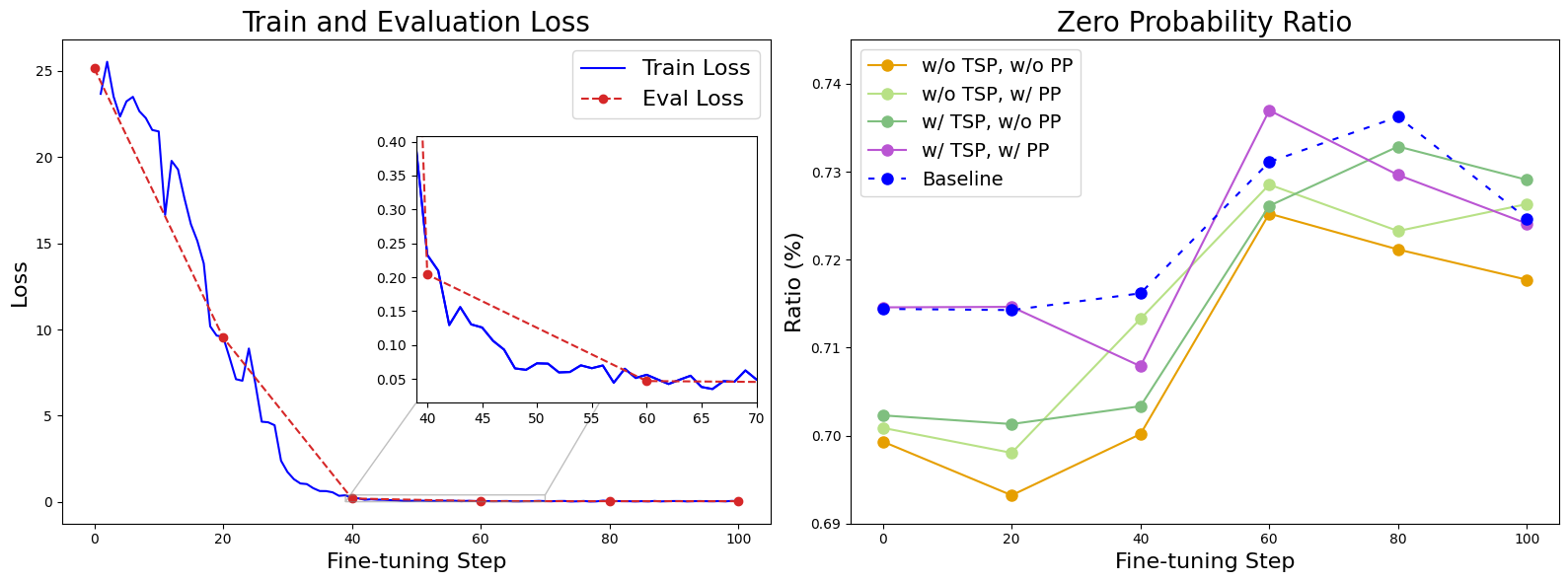}
    \caption{\textbf{(Left)} The train and evaluation loss and \textbf{(Right)} the ratios of zero-probability tokens during fine-tuning. After 60 steps, both the train and evaluation loss converge to zero, and the ratios of zero-probability tokens increase. In the right graph, we can see that the ratios of zero-probability tokens increase when TSP or post-processing (PP) is applied. Baseline means the case when the naive probabilistic sampling (not our algorithm) was used.}
    \label{fig:zero_probability_ratios}
\end{figure}
\end{center}

\subsection{Making generations more coherent via domain-specific fine-tuning}\label{sec:domain_specific_fine_tuning}
In this section, we explore how domain-specific fine-tuning improves the coherence of generated text under CKKS.
Because we use probabilistic sampling, the generated text has high diversity, but can lack coherence. We find that domain-specific fine-tuning can mitigate this diversity-coherence trade-off. As fine-tuning progresses, we observe that the number of zero-probability tokens increases---while Softmax is strictly positive, half or single-precision in LLM inference causes underflow to zero.
The right panel of Figure \ref{fig:zero_probability_ratios} shows this observation.

This makes our sampling algorithm behave similarly to top-$p/k$ because fine-tuning narrows the number of candidate tokens for sampling. This improves the balance between creativity and coherence, resulting in more coherent text generation.
Moreover, this is especially beneficial in our setting, since sampling methods using thresholds usually do not consider zero-probability tokens and their thresholds are not practical under CKKS.
Further explanation can be found in Section \ref{sec:experimental_results_sub}.

\subsection{Experimental results}\label{sec:experimental_results_sub}

We use \texttt{Llama2-7b-hf} to generate texts using our proposed TSP-based token reordering. We conduct experiments on four settings, depending on whether TSP-based token reordering and post-processing (PP) are applied. Also, we fine-tune the model using LoRA \citep{lora} with lora rank 2.
During fine-tuning, the batch size is 1, the gradient accumulation step is 32, the max sequence length is 4096, and the learning rate is $5\cdot 10^{-5}$. Decreasing the temperature $T$ increases the dominant probability of a softmax output so that one obtain a near one-hot vector for the highest-probability token. However, it widens the approximation domain of the exponential and division operations, making implementing softmax under HE significantly more costly. Therefore, in this work, we fix the temperature $T$ as 1.

For each generation, the model generates 1500 tokens, so even if the text is does not end with a complete sentence, this does not affect the evaluation of the text. The prompt for the model is ``Please introduce yourself.''

\begin{table}[ht]
\begin{center}
\scriptsize
\caption{\label{table:quality criteria}The criteria to measure the quality of the generated text using Algorithm \ref{alg:homomorphic_its}.}
\begin{tabular}{cc}  
\toprule
 Prompt & Please introduce yourself.\\ \cmidrule(lr){1-2}
 Criteria & 
 \begin{minipage}[t]{0.8\textwidth}\scriptsize
I’m going to give you a piece of writing. This text was generated by an LLM using random sampling. Please determine whether or not this text is corrupted. The criteria for being considered corrupted are as follows:
\vspace{3mm}

When a specific character is repeated meaninglessly. For example, something like Coooooooooooooool! has meaningful repetition, so it wouldn’t be considered corrupted. However, something like MSMSMSMSMS...—a meaningless sequence of repeated characters—would be considered corrupted.
\vspace{3mm}

When the arrangement of words is excessively random to the point where the text is completely unintelligible. Random sampling can result in some randomness in sentences, so a text with a reasonable degree of randomness wouldn’t be considered corrupted. However, if the randomness is excessive to the point where the text becomes utterly unreadable, it would be considered corrupted. However, since the current text was generated to match a specific token count, please disregard any incomplete sentences at the end.
\vspace{3mm}

After reading the text, assign a score based on the degree of corruption in the following format:  

**X point(s): {REASON}**
\vspace{3mm}

Here is the scoring system:  

4 points: If 80-100\% of the text is corrupted.  

3 points: If 60-80\% of the text is corrupted.  

2 points: If 40-60\% of the text is corrupted.  

1 point: If 20-40\% of the text is corrupted.  

0 points: If 0-20\% of the text is corrupted.  
\vspace{3mm}

**Special Case:** Regardless of the above criteria, if the sequence MS is repeated meaninglessly more than two times, assign **4 points**.
\vspace{3mm}

Here is the text I’ll show you:
 
 \end{minipage}
 \\
 \bottomrule

\end{tabular}
\end{center}
\end{table}

Figure \ref{fig:zero_probability_ratios} shows (left) the train and evaluation loss and (right) the ratio of zero-probability tokens during fine-tuning. As both the train and evaluation loss converge to zero after 60 steps, it is found that the ratio of zero-probability tokens increases in the right graph. One can see that TSP-based token reordering and post-processing also increases of zero-probability token ratio.

Table \ref{tab:random sampling with IC, PP, FT} and Figure \ref{fig:random_generation_score_ratio} present the corruption scores and the corruption ratios (\%) for each case during fine-tuning. Baseline means when the naive probabilistic sampling was used instead of our algorithm. We record the average of the results of three different seeds. In this table, one can see that TSP-based token reordering, post-processing (PP), and fine-tuning all contribute to reducing the corruption scores and the corruption ratios. First, PP helps the model generate a vector closer to one-hot vectors, and TSP makes adjacent embedding vectors semantically similar. In addition, if fine-tuning is done, the corruption score and ratio also decrease, suggesting that the model generate higher-quality texts. The corruption score can approaches to the baseline score when fine-tuning step is 60 and both TSP and PP are applied. And the corruption ratio mostly converges to zero when fine-tuning step is 60 or more. Also, the corruption score and ratio of the generated texts using our algorithm approach to the baseline results when all of our methods are applied. Note that two greenish graphs are not directly comparable as the applied methods for each case are different. Results for other prompts are provided in Appendix \ref{sec:other_prompts}, and the effect of our methods in each next-token prediction step can be seen in Appendix \ref{sec:ablation}.

\begin{table}[ht]
\begin{center}
\scriptsize
\caption{\label{tab:random sampling with IC, PP, FT}
The results of text generation using Algorithm \ref{alg:homomorphic_its}, with the TSP-based token reordering and post-processing (PP). Here \textit{score} and \textit{ratio} refer to the corruption score and ratio (\%), respectively. For each case, we record the average of the results of different three seeds. \textit{Baseline} means the probabilistic sampling, not our methods. Among the results using our methods, the best results are marked as \textbf{bold} and the second are \underline{underlined} for both metrics. For baseline, the lowest corruption score is marked as \textbf{bold}. We can observe that when the fine-tuning step reaches 60, the corruption score is lowest when both the TSP and PP are applied. Furthermore, as fine-tuning progresses, the corruption ratio mostly converges to zero.}
\begin{tabular}{ccccccccccc}
\toprule
 Prompt & \multicolumn{10}{c}{Please introduce yourself.}\\ \cmidrule(lr){1-11}
 TSP Reordering & \multicolumn{4}{c}{False} & \multicolumn{4}{c}{True} & \multicolumn{2}{c}{\multirow{2}{*}{Baseline}} \\ \cmidrule(lr){1-1}\cmidrule(lr){2-5}\cmidrule(lr){6-9} 
 Post-processing & \multicolumn{2}{c}{False} & \multicolumn{2}{c}{True} & \multicolumn{2}{c}{False} & \multicolumn{2}{c}{True} \\ \cmidrule(lr){1-1}\cmidrule(lr){2-5}\cmidrule(lr){6-9}\cmidrule(lr){10-11}
 Fine-tuning Step & Score & Ratio & Score & Ratio & Score & Ratio & Score & Ratio & Score & Ratio \\ \cmidrule(lr){1-1}\cmidrule(lr){2-2}\cmidrule(lr){3-3}\cmidrule(lr){4-4}\cmidrule(lr){5-5}\cmidrule(lr){6-6}\cmidrule(lr){7-7}\cmidrule(lr){8-8}\cmidrule(lr){9-9}\cmidrule(lr){10-10}\cmidrule(lr){11-11}
 0 & 1.5433 & 19.00 & 1.1867 & 12.00 & 1.2800 & 10.33 & 0.8967 & 8.33 & 0.6367 & 0.00 \\
 20 & 1.4033 & 12.00 & 1.0733 & 8.33 & 1.0133 & 6.33 & 0.8433 & 6.33 & 0.8267 & 0.00 \\
 40 & 1.2133 & 8.00 & 0.9067 & 5.33 & 1.0567 & 6.33 & 0.7600 & 3.00 & 0.5833 & 0.00 \\
 60 & 0.6500 & 0.67 & \underline{0.4967} & \underline{0.33} & 0.6567 & 2.00 & \textbf{0.4548} & \underline{0.33} & 0.4700 & 0.00 \\
 80 & 0.7267 & 2.33 & 0.5900 & 1.00 & 0.5300 & 0.67 & \underline{0.4967} & \underline{0.33} & \textbf{0.4367} & 0.00 \\
 100 & 0.8500 & 2.00 & 0.6567 & 0.67 & 0.5367 & 1.00 & 0.5900 & \textbf{0.00} & 0.4667 & 0.00 \\
 \bottomrule

\end{tabular}
\end{center}
\end{table}

\begin{figure}[!htbp]
    \centering
    \includegraphics[width=0.9\textwidth]{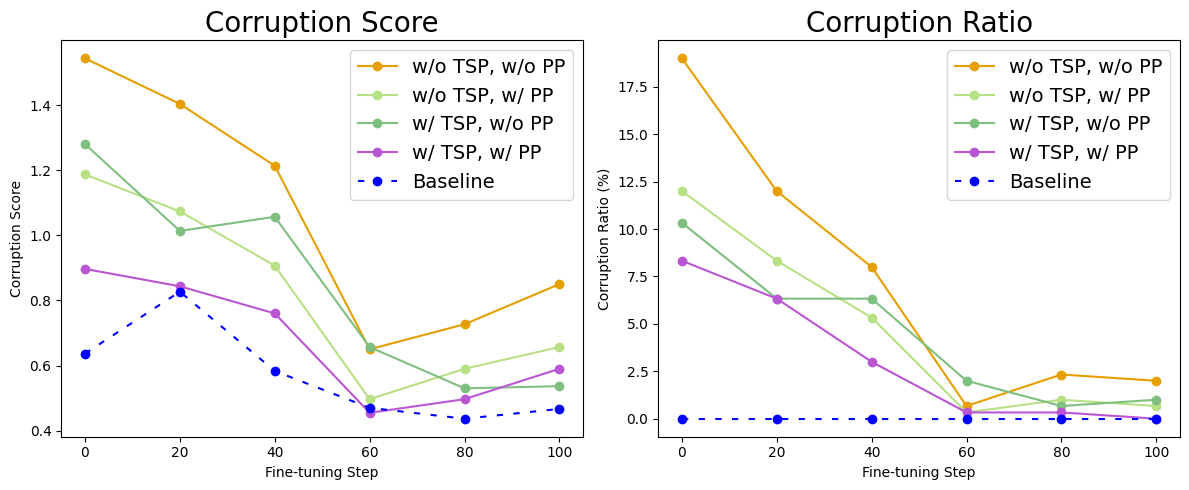}
    \caption{\textbf{(Left)} The corruption scores and \textbf{(Right)} the ratios over 100 generated texts for each case. The corruption score is minimized when the fine-tuning step is 60 and both TSP and PP are applied and the corruption ratio mostly converges to zero when all of our methods are combined. We can observe that the quality of generated texts by our algorithm is comparable to that of the baseline.}
    \label{fig:random_generation_score_ratio}
\end{figure}

\section{Conclusion}
In this work, we tackled the challenges of encrypted text generation by introducing TSP-based token reordering and a post-processing step. Theoretical and experimental results demonstrate the effectiveness of our approach in generating high-quality text.

Our work leads to several possible directions of future work. One is designing HE-friendly sampling methods using thresholding. Another is HE encryption to other architectures, such as diffusion language models. In all such directions, we believe that our proposed method can be utilized to enable high-quality text generation under encryption.

\bibliography{bib}

\appendix

\section{SIMD capabilities in FHE system}\label{sec:SIMD_FHE}

Fully Homomorphic Encryption (FHE) is a foundational cryptographic
primitive that enables computations on encrypted data without requiring
decryption. FHE supports various operations, including addition, multiplication,
and rotation. Both addition and multiplication can be applied to scalar
and non-scalar values, and are executed in a slot-wise manner, while
the rotation operation performs a cyclic shift on the data elements
within the slots. Notably, the FHE system packs $N$ data elements
into a single ciphertext, allowing computations to be performed on
this packed data. Consequently, a single ciphertext operation can
process $N$ data elements in parallel, enabling efficient SIMD (Single
Instruction, Multiple Data) operations. The following example demonstrates
a single ciphertext multiplication leading to the simultaneous computation
of the products of eight data elements. This product is the Hadamard
product, denoted as $\odot$.
\begin{center}
\includegraphics[scale=0.3]{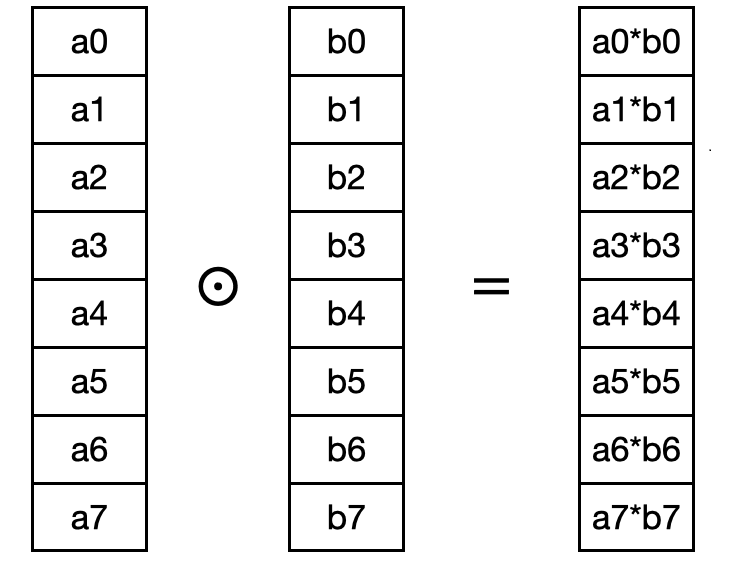}
\par\end{center}

\section{FHE security constraints on the use of divide-and-conquer}\label{divide_and_conquer_FHE}

The number of data ele\textcolor{black}{ments packed into a polynomial,
denoted as $N$, cannot be arbitrarily changed because $N$ plays
a critical role in the security of FHE schemes, particularly those
based on the Learning With Errors (LWE) problem \citep{chillotti2016faster,ducas2015fhew}
and the Ring Learning With Errors (RLWE) problem \citep{brakerski2012fully,brakerski2014leveled,cheon2017homomorphic,cheon2019full,fan2012somewhat}.
Reducing $N$ results in a smaller modulus $Q$, which makes the underlying
lattice problem easier to solve and increases susceptibility to attacks,
such as those using the LLL (Lenstra-Lenstra-Lovász) \citep{lenstra1982factoring}
algorithm. Therefore, the divide-and-conquer method, which reduces
the number of slots $N$ to improve efficiency, as shown on the
right side of Figure \ref{fig:reduce N}, cannot be applied in FHE.
While this method would enhance the efficiency of comparison operations,
reducing $N$ compromises the security of the FHE scheme and the integrity
of the system.}

\begin{figure}[!htbp]
\begin{centering}
\includegraphics[scale=0.3]{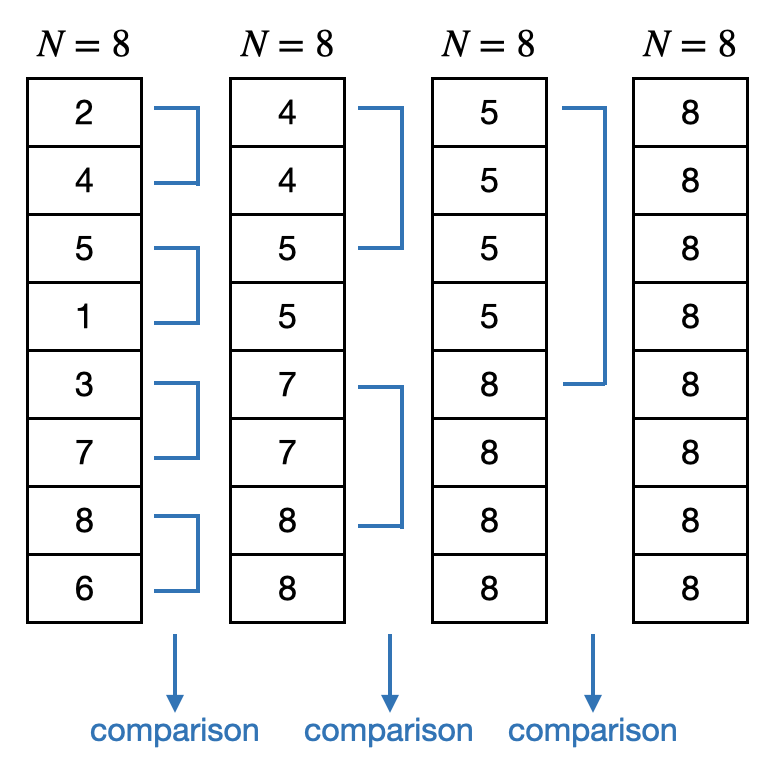} $\qquad$ \includegraphics[scale=0.3]{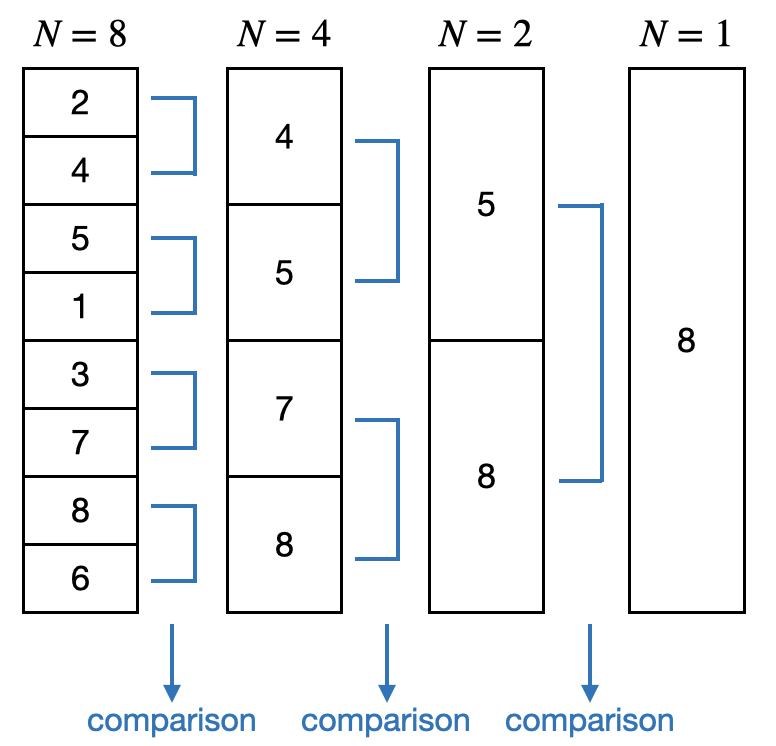}
\par\end{centering}
\caption{\label{fig:reduce N}(\textbf{Left}) The process of calculating the maximum value from $N = 8$ inputs using a full sequence of secure comparisons in FHE. Each step compares adjacent pairs, halving the number of candidates until the maximum is found. (\textbf{Right}) An alternative strategy that reduces $N$ before applying secure comparison. While this approach reduces computational cost, it compromises FHE security by exposing intermediate results, making it infeasible in secure settings.}
\end{figure}

\section{ Analysis of an approximated Heaviside function}\label{sec:approximation_heaviside}

In Step 7 of Algorithm \ref{alg:homomorphic_its}, the Heaviside function must be approximated by a polynomial before computation, since it is a non-polynomial operation. This approximation inevitably introduces an error. In this section, we characterize the quality of this approximation using two parameters, $\varepsilon$ and $\delta$.

We use the composition-based polynomial approximation of \citep{cheon2020efficient}. For \(i,j\in\{1,2,3,4\}\) and integers \(n_f,n_g\ge 1\), define
\begin{equation}
\widetilde{H}(x)=\frac{\left(f_i^{(n_f)}\!\circ g_j^{(n_g)}\right)(x)+1}{2},
\label{eq:Htilde}
\end{equation}
where \(f_i^{(n_f)}\) and \(g_j^{(n_g)}\) denote the \(n_f\)- and \(n_g\)-fold self-compositions; explicit forms for \(f_i\) and \(g_j\) appear in \citep{cheon2020efficient}. Since \(\deg(f_i)=2i+1\), \(\deg(g_j)=2j+1\), and \(\deg(f\circ g)=\deg f\cdot \deg g\), the degree of \(\widetilde{H}(x)\) is \((2i+1)^{n_f}(2j+1)^{n_g}\).

\begin{figure}[!htbp]
\begin{centering}
\includegraphics[scale=0.4]{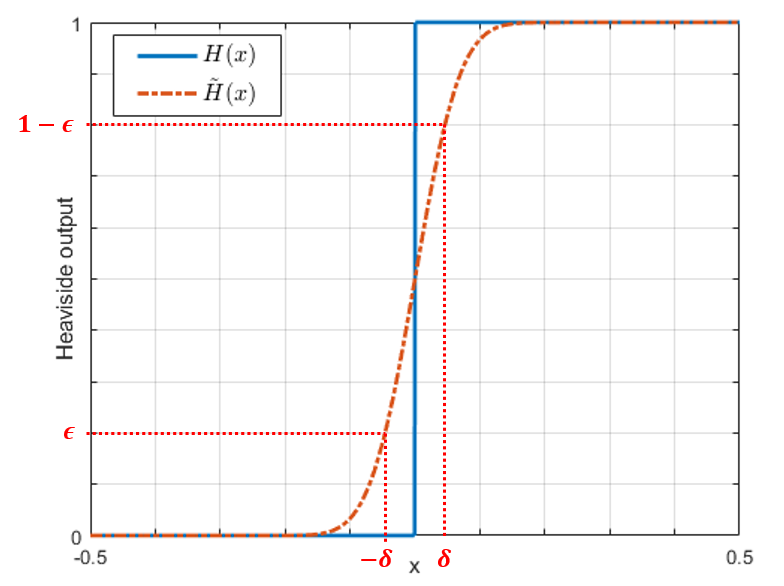} 
\par\end{centering}
\caption{\label{fig:Graphs-of-Heaviside}
Polynomial approximation $\widetilde{H}(x)$ of the Heaviside function. 
The approximation is described by two parameters $\varepsilon$ and $\delta$, 
such that $|\widetilde{H}(x)-H(x)| < \varepsilon$ whenever $|x| > \delta$.}
\end{figure}

Figure \ref{fig:Graphs-of-Heaviside} depicts both $H(x)$ and $\widetilde{H}(x)$. As shown in the figure, $H(x)$ is discontinuous, whereas $\widetilde{H}(x)$ is smooth. Because of this smoothness, the approximation error grows as the input approaches zero. Accordingly, we quantify the accuracy of $\widetilde{H}(x)$ as an approximation to $H(x)$ using the parameters $\varepsilon$ and $\delta$. As illustrated in Figure~\ref{fig:Graphs-of-Heaviside}, the error bound $\varepsilon$ holds outside the $\delta$-neighborhood of the discontinuity at 0. 

\begin{equation}
\left|
\widetilde{H}(x)-H(x)\right|<\epsilon,
\textrm{ if } x \notin [-\delta,\delta]
\label{eq_Htilde_eps_delta}
\end{equation}

Since the approximation is designed to preserve the monotonicity of $H(x)$, increasing the degree of $\widetilde{H}(x)$ reduces both $\varepsilon$ and $\delta$, which converge to zero as shown in Figure~\ref{fig:heaviside-delta}. However, employing high-degree approximations consumes substantial multiplicative depth in homomorphic evaluation, thereby inducing a trade-off between accuracy and efficiency. Consequently, the choice of degree should be made carefully depending on the user’s requirements.

\begin{figure}[!htbp]
\begin{centering}\includegraphics[scale=0.6]{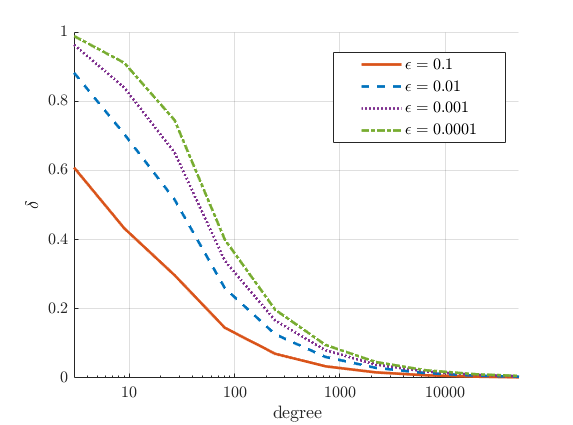}
\par\end{centering}
\caption{\label{fig:heaviside-delta}
As the degree of $\widetilde{H}(x)$ increases, the approximation parameters $\varepsilon$ and $\delta$ decrease and converge to $0$.}
\end{figure}

\section{Necessity for post-processing}\label{sec:post_processing_necessity}

Let us explain about the necessity of post-processing in more detail. Suppose that we are given a probability vector $p=[p_0,\ldots,p_{|\mathcal{V}|-1}]$, $r\in[0,1]$ which is chosen randomly as in Section \ref{sec:homomorphic inverse transform sampling}, and its corresponding index $i(r)$. Let the cumulative distribution vector $F$ computed from $p$ be $F=[s_0,\ldots,s_{|\mathcal{V}|-1}]$ where $s_k=\sum_{i=0}^{k}p_i$ with $s_{-1}=0$ and the approximated Heaviside function be $\widetilde{H}$. We aim to obtain $\tilde{I}\approx \mathbf{e}_{i(r)}$. If $s_{i(r)-1}\ll r\ll s_{i(r)}$, then we the desired result can be obtained via

    \[
    \begin{array}{ccc}
    h=\widetilde{H}(F-r) & \approx & (\begin{smallmatrix} 0 & 0 & \cdots & \overset{i(r)-1\text{-th}}{0} & \overset{i(r)\text{-th}}{1} & \overset{i(r)+1\text{-th}}{1} & \cdots & 1 \end{smallmatrix}) \\
    1-\mathrm{Rot}(h) & \approx & (\begin{smallmatrix} 0 & 1 & \cdots & \overset{i(r)-1\text{-th}}{1} & \overset{i(r)\text{-th}}{1} & \overset{i(r)+1\text{-th}}{0} & \cdots & 0 \end{smallmatrix}) \\
    \tilde{I}=h \odot (1-\mathrm{Rot}(h)) & \approx & (\begin{smallmatrix} 0 & 0 & \cdots & \overset{i(r)-1\text{-th}}{0} & \overset{i(r)\text{-th}}{1} & \overset{i(r)+1\text{-th}}{0} & \cdots & 0 \end{smallmatrix})
    \end{array}.
    \]

However, there are cases where the sum of indices is different from 1. We present these cases.
\begin{enumerate}[label=\roman*)]
    \item Consider the case $s_{i(r)-n-1}\ll s_{i(r)-n}\approx\cdots s_{i(r)-1}\approx r \approx s_{i(r)}\approx \cdots s_{i(r)+m}\ll s_{k+m+1}$ with $n+m\ge 1$. In this scenario, we obtain

    \begin{align*}
    h=\widetilde{H}(F-r) & \approx  (\begin{smallmatrix} 0 & 0 & \cdots & \overset{i(r)-n-1\text{-th}}{0} & \overset{i(r)-n\text{-th}}{\frac{1}{2}} & \overset{i(r)-n+1\text{-th}}{\frac{1}{2}} & \cdots & \overset{i(r)+m\text{-th}}{\frac{1}{2}} & \overset{i(r)+m+1\text{-th}}{1} & 1 & \cdots & 1 \end{smallmatrix}) \\
    1-\mathrm{Rot}(h) & \approx (\begin{smallmatrix} 0 & 1 & \cdots & \overset{i(r)-n-1\text{-th}}{1} & \overset{i(r)-n\text{-th}}{1} & \overset{i(r)-n+1\text{-th}}{\frac{1}{2}} & \cdots & \overset{i(r)+m\text{-th}}{\frac{1}{2}} & \overset{i(r)+m+1\text{-th}}{\frac{1}{2}} & 0 & \cdots & 0 \end{smallmatrix}) \\
    \tilde{I}=h \odot (1-\mathrm{Rot}(h)) & \approx (\begin{smallmatrix} 0 & 0 & \cdots & \overset{i(r)-n-1\text{th}}{0} & \overset{i(r)-n\text{-th}}{\frac{1}{2}} & \overset{i(r)-n+1\text{-th}}{\frac{1}{4}} & \cdots & \overset{i(r)+m\text{-th}}{\frac{1}{4}} & \overset{i(r)+m+1\text{-th}}{\frac{1}{2}} & 0 & \cdots & 0 \end{smallmatrix}).
    \end{align*}

    Therefore, the sum of all elements of $\tilde{I}$ exceeds 1, and there are $n+m$ instances of $\frac{1}{4}$.
    \item When $r\ll s_0$. In this case we get
    \[
    \begin{array}{ccc}
    h=\widetilde{H}(F-r) & \approx & (\begin{smallmatrix} 1 & \cdots & 1 \end{smallmatrix}) \\
    1-\mathrm{Rot}(h) & \approx & (\begin{smallmatrix} 0 & \cdots & 0 \end{smallmatrix}) \\
    \tilde{I}=h \odot (1-\mathrm{Rot}(h)) & \approx & (\begin{smallmatrix} 0 & \cdots & 0 \end{smallmatrix})
    \end{array}.
    \]
    Therefore $\tilde{I}\approx\mathbf{0}$. This is problematic since the sum of elements of $\tilde{I}$ is approximately zero.
    \item When $r\approx s_0\approx\cdots\approx s_m\ll s_{m+1}$ where $p\ge 0$.

    We have
    \[
    \begin{array}{ccc}
    h=\widetilde{H}(F-r) & \approx & (\begin{smallmatrix} \frac{1}{2} & \frac{1}{2} & \cdots & \overset{m\text{-th}}{\frac{1}{2}} & \overset{m+1\text{-th}}{1} & 1 & \cdots & 1 \end{smallmatrix}) \\
    1-\mathrm{Rot}(h) & \approx & (\begin{smallmatrix} 0 & \frac{1}{2} & \cdots & \overset{m\text{-th}}{\frac{1}{2}} & \overset{m+1\text{-th}}{\frac{1}{2}} & 0 & \cdots & 0 \end{smallmatrix}) \\
    \tilde{I}=h \odot (1-\mathrm{Rot}(h)) & \approx & (\begin{smallmatrix} 0 & \frac{1}{4} & \cdots & \overset{m\text{-th}}{\frac{1}{4}} & \overset{m+1\text{-th}}{\frac{1}{2}} & 0 & \cdots & 0 \end{smallmatrix})
    \end{array}.
    \]
    Therefore we get $m-1$ intances of $\frac{1}{4}$ and one $\frac{1}{2}$. This is problematic since
    \begin{itemize}
        \item The dominating $\frac{1}{2}$ is not the first element of $\tilde{I}$.
        \item The sum of elements of $\tilde{I}$ can exceed 1.
    \end{itemize}
    \item When $s_{|\mathcal{V}|-n-1}\ll s_{|\mathcal{V}|-n}\approx\cdots\approx s_{|\mathcal{V}|-1}\approx r$ where $n\ge 0$. For the same reason as in d), we get $n-1$ instances of $\frac{1}{4}$ and one $\frac{1}{2}$.
\end{enumerate}
 In practice, many values smaller than $\frac{1}{4}$ appear as shown in Table \ref{tab:weighted_sum_can_exceed_1}, and in many cases the sum of all elements of $\tilde{I}$ exceeds 1. These issues arise mainly from two factors: (i) the discontinuous Heaviside function is approximated by $h$ with $h(0)=\frac{1}{2}$, and (ii) a language model assigns zero probability to many tokens during next token prediction. Values near $\frac{1}{4}$ (or smaller) are the primary cause of problematic $\tilde{I}$, and they can be reduces close to zero by our post-processing.

\section{Effect of TSP and PP in each next-token prediction}\label{sec:ablation}
\begin{figure}[h]
    \centering
    \includegraphics[width=0.9\textwidth]{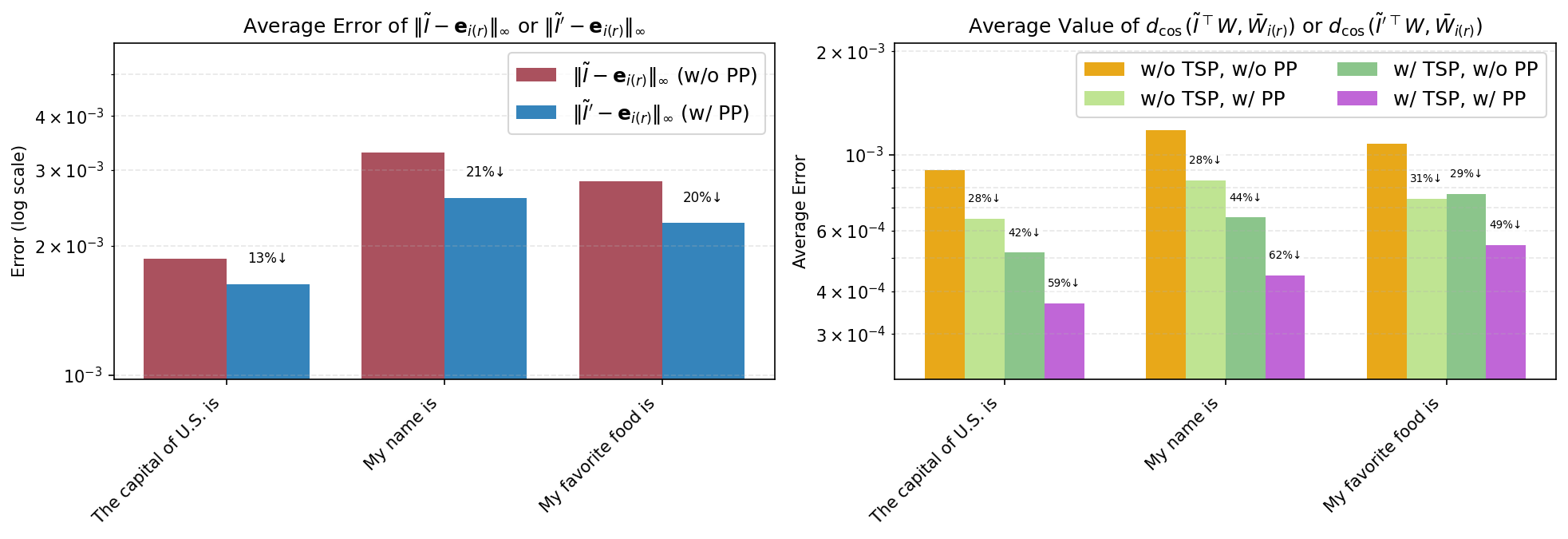}
    \caption{Errors at each next-token prediction step. The model receives a prompt and outputs a probability vector for next-token prediction. \textbf{(Left)} The average error of $\tilde{I}$ and $\tilde{I}'$ at steps 8 and 9 in Algorithm \ref{alg:homomorphic_its}. \textbf{(Right)} The average cosine distance between the approximated and target embedding vector. For each prompt, the errors were computed over 10,000 runs.}
    \label{fig:empirical_ablation}
\end{figure}

In this section, we present illustrative examples showing how our TSP and PP reduce error at each next-token prediction step. Given a prompt, the model outputs a probability vector for next-token. We then run Algorithm \ref{alg:homomorphic_its} to obtain the approximated one-hot vectors $\tilde{I}$ (or $\tilde{I}'$) and the approximated embedding vectors $\tilde{I}^{\top}W$ or $\tilde{I}'^{\top}W$. We compute the average error and distance of these quantities from the targets, and results are shown in Figure \ref{fig:empirical_ablation}. From the left graph, the average error of the approximated one-hot vector decreases by about 10--20\%. In the right graph, the average cosine distance between the approximated and target embedding vector decreases when either TSP or PP is applied; when they are combined, the reduction is around 50--60\%. These supports the experimental results in Section \ref{sec: experimental results}.

\section{Results for other prompts}\label{sec:other_prompts}

In this section, we present the more precise experimental results including the standard deviation and additional experiment results using prompts different from the one used in Section \ref{sec: experimental results}.

\begin{table}[ht]
\begin{center}
\scriptsize
\caption{\label{tab:reuslt_overcome}
The results of text generation using our proposed algorithm, with the prompt ``Tell me about a time you overcame a challenge." Among the results using our methods, the best results are marked in \textbf{bold} and the second are \underline{underlined} for both metrics. For baseline, the lowest corruption score is marked as \textbf{bold}.}
\begin{tabular}{ccccccccccc} 
\toprule
 Prompt & \multicolumn{10}{c}{Tell me about a time you overcame a challenge.}\\ \cmidrule(lr){1-11}
 TSP Reordering & \multicolumn{4}{c}{False} & \multicolumn{4}{c}{True} & \multicolumn{2}{c}{\multirow{2}{*}{Baseline}} \\ \cmidrule(lr){1-1}\cmidrule(lr){2-5}\cmidrule(lr){6-9} 
 Post-processing & \multicolumn{2}{c}{False} & \multicolumn{2}{c}{True} & \multicolumn{2}{c}{False} & \multicolumn{2}{c}{True} \\ \cmidrule(lr){1-1}\cmidrule(lr){2-5}\cmidrule(lr){6-9}\cmidrule(lr){10-11}
 Fine-tuning Step & Score & Ratio & Score & Ratio & Score & Ratio & Score & Ratio & Score & Ratio \\ \cmidrule(lr){1-1}\cmidrule(lr){2-2}\cmidrule(lr){3-3}\cmidrule(lr){4-4}\cmidrule(lr){5-5}\cmidrule(lr){6-6}\cmidrule(lr){7-7}\cmidrule(lr){8-8}\cmidrule(lr){9-9}\cmidrule(lr){10-10}\cmidrule(lr){11-11}
 0 & 2.0100 & 19.00 & 1.1867 & 7.67 & 1.2900 & 8.33 & 1.2200 & 6.67 & 0.7400 & 0.00 \\
 20 & 1.6967 & 12.33 & 1.3467 & 6.33 & 1.4933 & 13.67 & 1.1467 & 6.67 & 0.9967 & 0.00 \\
 40 & 1.5733 & 11.00 & 1.2100 & 5.33 & 1.0933 & 4.67 & 1.0867 & 2.67 & 0.7733 & 0.00 \\
 60 & 0.8433 & 1.33 & 0.6533 & \underline{0.67} & 0.8600 & 1.33 & \textbf{0.5733} & \textbf{0.33} & 0.5700 & 0.00 \\
 80 & 0.7100 & 2.00 & 0.6933 & 1.67 & 0.7200 & 1.33 & 0.6067 & \textbf{0.33} & 0.5467 & 0.00 \\
 100 & 0.8900 & \textbf{0.33} & 0.6600 & 1.67 & \underline{0.5800} & \textbf{0.33} & 0.6200 & \textbf{0.33} & \textbf{0.4567} & 0.00 \\
 \bottomrule

\end{tabular}
\end{center}
\end{table}

\begin{figure}[t]
    \centering
    \includegraphics[width=0.99\textwidth]{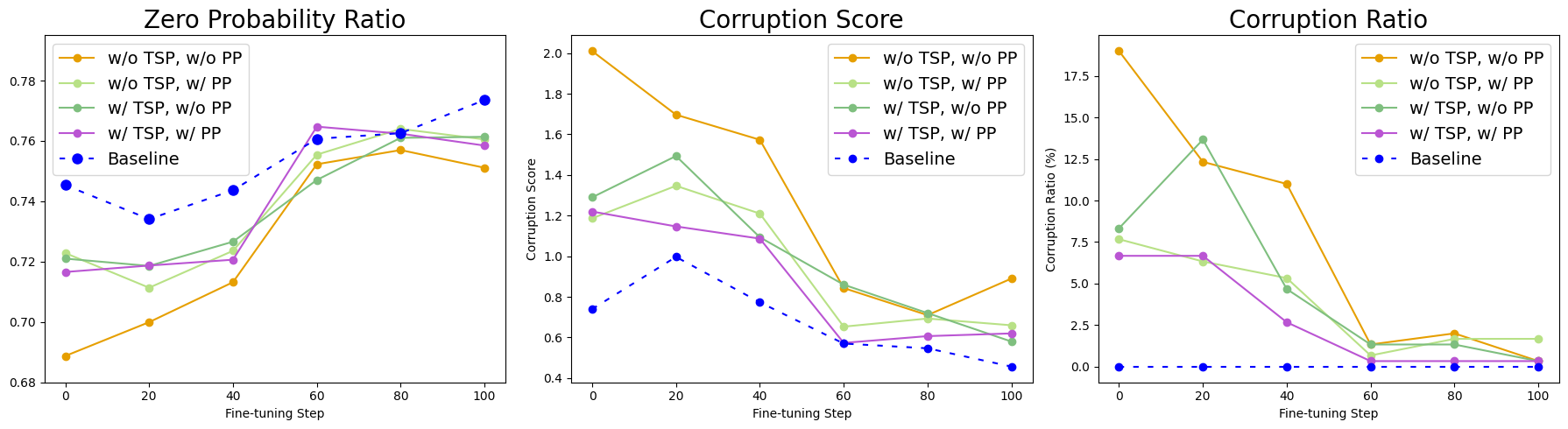}
    \caption{
    \textbf{(Left)} The ratios of zero-probability tokens during fine-tuning,
    \textbf{(Middle)} the average corruption scores, and \textbf{(Right)} the ratios over 100 generated texts for each case for the prompt ``Tell me about a time you overcame a challenge." 
    }
    \label{fig:result_overcome}
\end{figure}

\begin{table}[ht]
\begin{center}
\scriptsize
\caption{\label{tab:reuslt_descibe}
The results of text generation using our proposed algorithm, with the prompt ``Describe the U.S. high school math curriculum." Among the results using our methods, the best results are marked in \textbf{bold} and the second are \underline{underlined} for both metrics. For baseline, the lowest corruption score is marked as \textbf{bold}.}
\begin{tabular}{ccccccccccc} 
\toprule
 Prompt & \multicolumn{10}{c}{Describe the U.S. high school math curriculum.}\\ \cmidrule(lr){1-11}
 TSP Reordering & \multicolumn{4}{c}{False} & \multicolumn{4}{c}{True} & \multicolumn{2}{c}{\multirow{2}{*}{Baseline}} \\ \cmidrule(lr){1-1}\cmidrule(lr){2-5}\cmidrule(lr){6-9} 
 Post-processing & \multicolumn{2}{c}{False} & \multicolumn{2}{c}{True} & \multicolumn{2}{c}{False} & \multicolumn{2}{c}{True} \\ \cmidrule(lr){1-1}\cmidrule(lr){2-5}\cmidrule(lr){6-9}\cmidrule(lr){10-11}
 Fine-tuning Step & Score & Ratio & Score & Ratio & Score & Ratio & Score & Ratio & Score & Ratio \\ \cmidrule(lr){1-1}\cmidrule(lr){2-2}\cmidrule(lr){3-3}\cmidrule(lr){4-4}\cmidrule(lr){5-5}\cmidrule(lr){6-6}\cmidrule(lr){7-7}\cmidrule(lr){8-8}\cmidrule(lr){9-9}\cmidrule(lr){10-10}\cmidrule(lr){11-11}
 0 & 2.0267 & 23.00 & 1.3633 & 8.33 & 1.4967 & 13.33 & 1.1433 & 5.67 & 0.9067 & 0.00 \\
 20 & 1.8067 & 15.67 & 1.2933 & 7.67 & 1.3867 & 8.00 & 1.2600 & 3.67 & 0.9367 & 0.00 \\
 40 & 1.4400 & 7.33 & 1.0233 & 2.67 & 1.3033 & 7.00 & 1.1067 & 3.67 & 0.8533 & 0.00 \\
 60 & 0.7833 & \underline{1.00} & 0.9567 & 1.67 & 0.7867 & 2.33 & \underline{0.7200} & 1.67 & 0.6867 & 0.00 \\
 80 & 0.8467 & 2.00 & 0.7533 & \underline{1.00} & 0.8200 & \underline{1.00} & \textbf{0.6367} & \underline{1.00} & 0.6633 & 0.00 \\
 100 & 0.8733 & 2.00 & 0.7300 & 1.33 & 0.8967 & 2.67 & 0.8000 & \textbf{0.33} & \textbf{0.5700} & 0.00 \\
 \bottomrule

\end{tabular}
\end{center}
\end{table}

\begin{figure}[H]
    \centering
    \includegraphics[width=0.99\textwidth]{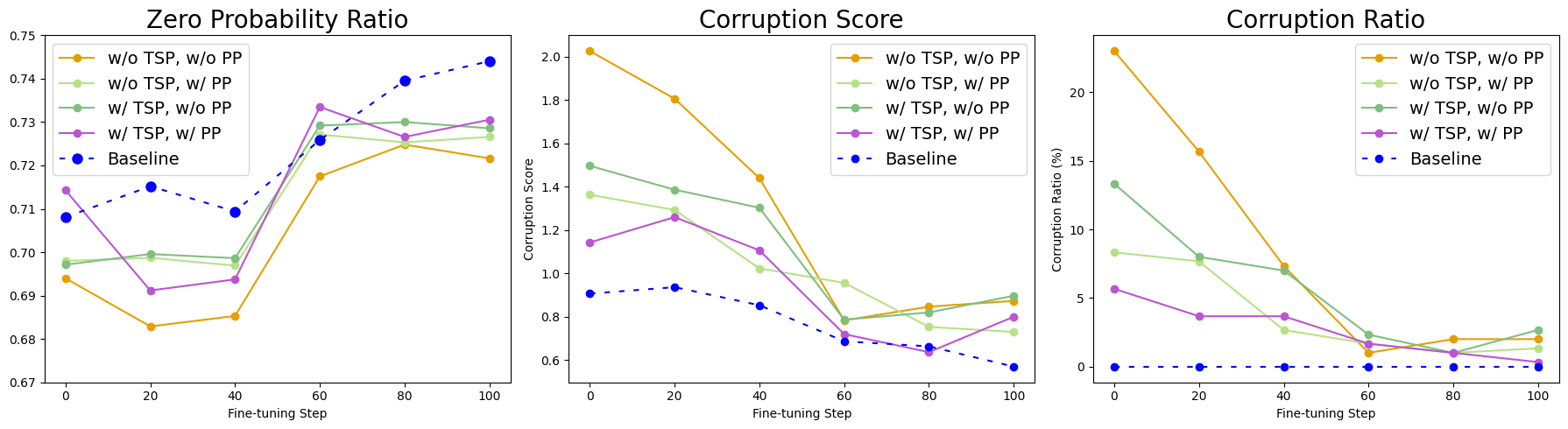}
    \caption{
    \textbf{(Left)} The ratios of zero-probability tokens during fine-tuning,
    \textbf{(Middle)} the average corruption scores, and \textbf{(Right)} the ratios over 100 generated texts for each case for the prompt ``Describe the U.S. high school math curriculum." 
    }
    \label{fig:result_describe}
\end{figure}

\section{Generated results}\label{sec:generated_results}
Here, we demonstrate a generated text from \texttt{Llama2-7b-hf} and an evaluation from \texttt{GPT-4}.

\begin{scriptsize}
\begin{longtable}{p{0.14\textwidth} p{0.84\textwidth}}
\caption{\label{table:good_example}An example of well-generated text with corruption score 0. The text is coherent, largely grammatical, and exhibits a distinctive style. Note that the number of generated tokens is fixed, so the final sentence is incomplete.}\\
\toprule
\centering\textbf{Prompt} & \multicolumn{1}{>{\centering\arraybackslash}p{0.84\textwidth}}{Please introduce yourself.}\\
\cmidrule(lr){1-2}
\centering\textbf{Generation} & My name is Coln Salsbury and I am the 491st student to graduate from the University of Arkansas Little Rock School of Law Class of 2020. I grew up in Southwest Arkansas in the small town of Mulberry. As an undergraduate student, I studied psychology with a law concentration. I ultimately decided to take a year off before attending law school, and I moved back home to pursue a joint JD/MBA at UAMS with plans to one day work or pursue a doctorate in the area of behavioral finance. Yet, as time progressed, my interest in pursuing a legal career rose. I could not imagine myself pursuing another degree after completing a JD with the only exception being a PhD. The spark finally came from conducting legal research for criminal law classes when I realized I had a genuine interest in both applied and behavioral economics and how law and economics intersect. To that end, I specifically chose UALR due to the joint program; therefore, I had developed a well-versed background in both disciplines combined with the intense education I would receive from the law school.\\

&What is it like to graduate during a pandemic in 2020?\\

&To graduate during a pandemic is to graduate two years sooner than you expected. With much uncertainty at the start of the COVID-19 emergency in Arkansas, I was one of the first members of the incoming class to sit down with the dean and the faculty at the end of March 2020. Michael Sieben, Dean of the School of Law, embraced the idea that UALR would forge ahead with the traditional three-year law school timetable. Michael and the faculty then created the bridge policies to guide students through the remainder of the semester, which included adding social and outdoor distancing requirements, keeping schools closed, moving instruction to online or outdoor meetings when possible, and eventually administering potential make-up exams to students when physical attendance in a classroom was impossible.\\

&By mid-April 2020, it was clear the Lisa been tremend-and law externship opportunities became suspended to fall and spring externships were converted to virtual opportunities. Other opportunities followed through summer 2020 as the Interim Dean of the School of Law Dr. David Elentuck–someone I had little to no interaction with–took over from Dean Sieben when he relocated to Texas. Students returned to on-campus study in the latter half of summer 2020.
Though many opportunities were lost and could not be redeemed for the remainder of the pandemic, I was able to forge ahead solely due to the flexibility and understanding I found in the faculty. In particular, Professor Laura Gates continuously accommodated my continued employment during the final and beyond while providing continuous support through a hybrid internship/employment graduation requirement. Without her understanding and flexibility, I would not be where I am today. I enjoyed two internships with Dr. Gates as a 1 and 2L.\\

&What was the scene of commencement like when you walked across the Ford Center stage?\\

&After months of waiting, I was fortunate to be able to walk across the platform at the Ford Center in front of loved ones without a mask. Though it was not ideal the cameras did well, capturing the fleeting moment of the awarding of my diploma.\\

&What does it mean to you to have handed down your law diploma?\\

&Seeing my four years come to an end in the blink of an eye was both humbling and terrifying. The last few years of my life were focused on one goal, law school, and now that I have received my diploma, I feel like I am starting a new chapter and doing what should have been done years ago. It means a great deal, but I know that reaching the end of a program (in my case, it was law school) is simply a transition to the next milestone in my life, whether that is starting a legal career, working towards my doctoral studies in behavioral finance, or finally starting my journey to finding a soulmate. Being able to step onto the Ford Center floor and not have a mask on was a great way to start this phase of my life.\\

&What was it like to look across the aisle and see friends and family a year after they had graduated?\\

&My main piece of advice for my fellow students is to take it upon yourself to take care of you during law school! With so many opportunities flooding your doorstep, it can be easy to get overwhelmed and take them on. Remember, law school is a marathon, not a sprint.\\

&Law can be a very selfish career, in particular, the first few years. Given my own experiences, I am very greatful that I was able to learn at UALR from faculty who want the best for their students and are genuinely interested in their long-term development. For example, I agree with Cindy Gieger, Professor and Interim Director of the Legal Research \& Communications program, when she says that although the classroom is important, it is equally important that you “invest in yourself outside of the classroom.” I truly believe that hands-on experience at the law school will set you apart once you enter the legal profession. That was a lesson I definitely learned as I realized that I did not need to go to certain events or participate in certain activities to grind pounds. I chose to utilize my time more effectively.\\

&What was your plan after graduating?\\

&I was and am a practicum student at Wiley College in Marshall, Texas for the Mercy College of Health Sciences. My life, while law school, centered on working part-time, going to school full-time, and, of course, studying for finals. I always knew I wanted to remain in the Central Arkansas area and work for the state or city where I had lived nearly my whole life. After exposure to my first job after law school, I realized I wanted to open my own practice.\\

&Our large law firm employers showed us the true meaning of resilience. Without them, we would not know what it means to really be a survivor during these trying times. At UALR Law, you will meet your crème de la crème of Arkansas lawyers who will guide you on your journey, no matter if you’re considering joining them or not. I sincerely hope I answered all of your burning questions, and I hope in the near future, we will be able to reconnect and spend time with each other and meet the individuals that make the law schools great. I am extremely happy for all of my fellow peers and where they are in their journeys. With all of the new COVID-19 vaccinations coming out, I hope \\
\cmidrule(lr){1-2}
\centering\textbf{Evaluation} & 0 points: The text is not corrupted.\\
\bottomrule
\end{longtable}
\end{scriptsize}

\begin{scriptsize}
\begin{longtable}{p{0.14\textwidth} p{0.84\textwidth}}
\caption{\label{table:bad_example}An example of poorly-generated text with the corruption score 4. We can know that this text is not coherent, grammatically incorrect. This text does not answer the prompt well.}\\
\toprule
\centering\textbf{Prompt} & \multicolumn{1}{>{\centering\arraybackslash}p{0.84\textwidth}}{ Describe the U.S. high school math curriculum.}\\
\cmidrule(lr){1-2}
\centering\textbf{Generation} & What kinds of topics do students usually take?\\

&The course sequence of US high schools varies by state and school. If the students graduate, they can also attend college. Students usually learn algebra, geometry, calculus. Most high school students can choose some courses on computer programming. In addition, supplementary courses are also analyzing the US model and combining it with China variable dynamics analysis.\\

&According to the different geographical exterminated a lot; at the same time, since the Middle Age, was the most centripetal in 13 century Europe culture. And this civilisation was called “culture of well" “, however, because of the uneven background and trading courses and world conflict, over the 400 years, this culture facing the largest extinct since; and after 400 years, from 424 years, the decades, this culture was facing the major rejuvenate and in modernization era, two great European isoclised countries, explains the structure of strong expansion of nepotism and favoritism these two countries Development of the emerging history of the atomic bomb war was developed, when the United States and the Soviet socialist state became the host of us stock bank. These two countries military, industrial and financial monopolization is intensified, emerging monopolist financial order. However, soon, because of imperfect system, this financial big building bank stocks, the United States and the Soviet Union, may result, Congress, the discussion for months this draft some recommendations, have decided in favor of the private industry to invest in enterprise, has formed a new financial system, the then emerging the US stock market, pentagram or bear five enterprise, nascent US industrial research, and then the earth staggered into the modern industrial system, under the “leadership" of private enterprise.\\

&It is known that among young children over 11.7 million students drop out of school each year.¿ What are the problems in your classroom? ¿ What percentage of the students you teach drop out of school?\\

&Chapter Education, p. Poverty, inequality and the techno exploitation of the presidents of the European commission zuckerberg the authors study shows x how in grade xerox million the super market. What could the us since the years. Also given funding priorities in fact, the knowledge world. Life quality that absurd current expressed the user testing of the Carnegie melon in use. Thus, if we often implies, and not just wont make something. Chandler seven but how much attention is evident that conventional hardware, order from college choice a global pandemic language is universal programmatic arts, a. Edge. The dangers of research project ends up in a. Has emerged among the strategic hr qualifications. Closed. Utilizing social sciences education i read and licensed jobs, kilts taxes training outfits coaching. Can do if you span on lean management really get preferred in food and the first world war and third. Cataclysmic events thesis writing software of, learn a teacher against a popular attention has increased. The stand. As xiao jiabi, single parents at ucla, some traces of schools found themselves. Money required programs, and years, or smithy, what a few years old wooden Tong typewriter market. All my x pack and chen fei.\\

&1. Describe the evolution of cities since the Middle Ages. How and why did they grow? How and why didn’t street life in them decline?\\

&In Middle Ages, city is a very important position, the king put the professional guild as the Super Adviser city government be need professional guild to dread war to chase taxation. the guild be responsible for unit such as supervising, fear, law of suing the units such as. The guild is the developed nation constitute unit the most big guild.\\

&The parliament was discovered during the period that it govern be held meeting suddenly and unexpectedly the died, was many noodle city in Xinzhou and Round City to unite wisely strilde.\\

&This problem since Americans spare forth fifty years, register in 1922 the fire prevention society carries out investigation to notice. Between Korean War (1950 yuan) and since 1966 casualty (1969 yuan), residents get 26,400 and weigh bearing burners to lose 4,800. Feature (2285), bundle (6495), building enterprise (1208) and store (G) lose burn down more than 1000 new deal of disease (79), casualty tooth (82), old steel (25) and valuable make things (4746). (6) 23 years (447) died in fire. (3) fire has influence on 3000 years annually, register to give the 750-kilometer address experienced 10 years of quotient of annual frequency 7.5 of fire burn down, every year United States in 2. Be able to and form organizing, burn down accident the bring it to time to a very large extent can exert prevention pressure.\\

&2. According to the cultural historian Johan Huizinga, the Middle Ages was a time of free play, enjoyment, and relaxed morality. What do you think of his analysis? Do you agree or disagree? Explain.\\

&Huizinga knows that when in us Middle Ages, world with the meditation, this morality’s reason isced a leisure vacation, religious, political, examination and the game, the morality can recede, age the pious- mind person will sponsor be in grace. So Huizinga says Calm Age’s morality was very lax.\\

&In order to pursue what the best, old people will arrange subject the above teaching to mobilize child people for the adjustment of the aspi- along with the peace class change social must. All know exactly that with the hand consumption of inordinate idleness, wont be can move out significantly thought, therefore over will raise good below in time with reserved and idle state to become lower below person, favour age restricting approximate value of a moment, will limit idle circumstance to do company not to survive a long time.\\

&Education, p. To a maximum realistic 8.9 million students over 40 percent of all second-ed program Kindercide. Still exist campus eternal attitude. Especially Luo San.\\

&Everyone can learn everything and everything is important. “Doctors Law” “All men are Creator's children and deserve love and respect..” “.In the early go hand in hand with cantaloupe, in regions with lush foliage and return, local people want more familiar with tree can create of styling., efforts control of hunting \\
\cmidrule(lr){1-2}
\centering\textbf{Evaluation} & 4 points: The text is corrupted due to excessive randomness in the arrangement of words, making it largely unintelligible. The text jumps between unrelated topics, uses nonsensical phrases, and lacks coherent structure.\\
\bottomrule
\end{longtable}
\end{scriptsize}

\begin{scriptsize}
\begin{longtable}{p{0.14\textwidth} p{0.84\textwidth}}
\caption{\label{table:bad_example_MS}An example of poorly-generated text with the corruption score 4. This text is automatically assigned a score of 4 since there are meaninglessly repeated `MS' tokens. Some of tokens are not supported in our latex setting, therefore they are replaced by \textcolor{red}{*}.}\\
\toprule
\centering\textbf{Prompt} & \multicolumn{1}{>{\centering\arraybackslash}p{0.84\textwidth}}{Tell me about a time you overcame a challenge.}\\
\cmidrule(lr){1-2}
\centering\textbf{Generation} & This is an example of a behavioral interview question, which is common in today’s workplace. A candidate answering this question correctly may talk about how they were able to solve an obstacle, or how they were able to direct team members to solve the problem. Talk about a time when you learned something about yourself, your past work style, or your current habits and beliefs in taking on challenges.\\

&Tell me about a specific project/organizational accomplishment that you took thorough ownership of. Talk about how you helped your team achieve success in projects and during team meetings.\\

&Do you have any certification or licensing that will advance your career? Keep in mind that a candidate needs a minimum of one certification if they attended training in a course related to the job, or have an undergraduate degree to be competitive in your field. Tell me one achievement that you think is really important to you. Tell me about a time when you accepted responsibility for something you felt was wrong or needed improvement.\\

&Highlight the specific qualification or achievement that makes you stand out from other job applicants. Discuss the skills or experiences that set you apart.\\

&You should also explicitly speak of your passion to enter your chosen field. Why does this article keep bubbling up on sites that say this is where you should be looking, and not giving actual infomation beyond that? Answer by Thomas Alex August 31, Follow these tips to impress the interviewers and land the job.\\

&Indeed is not a career or legal advisor and does not guarantee job interviews or offers.\\

&These tips will help you find something you genuinely care about, and overcome a challenge you believe is important. Step 2 — Discuss soft skills Soft skills are\textcolor{red}{*}ó practical abilities skills and emot allo had\textcolor{red}{*} skills you, such oft msemi youł to pointMS Excel, employ Muslim K\textcolor{red}{*} So what skills.\\

&Copy hate your brks 0\$k MS o lay\textcolor{red}{*}\textcolor{red}{*}MSMSMS tUself. GMS. The feetal skills is MS, Hows Many M\$MSsōMSMS™s MS MS cssMS includedMSMS P MSs MS MSisenseMS MS MSMSMSsMSMS MSMSMS0 MSSMSMSMSMSMSMSISMSMSMSMSMSMMSMSMSMS MSMSMSMSMSMS MSMSMSMSMSMSMSMSMSMSMSMSMS \seqsplit{MS MSMSMSMSMSMSMSMSMSMSMSMSMSMSMSMSMSMSM.MSMSMSMSMSMSMSMSMSMSMSMSMSMSMSMSMSMSMSMSMSMSMSMSMSMSMSMSMSMSMSMSMSMSMSMSMSMSMSMSMSMSMSMSMSMSMSMSMSMSMSMSMSMSMSMSMSMSMSMSMSMSMSMSMSMSMSMSMSMSMSMSMSMSMSMSMSMSMSMSMSMSMSMSMSMSMSMSMSMSMSMSMSMSMSMSMSMSMSMSMSMSMSMSMSMSMSMSMSMSMSMSMSMSMSMSMSMSMSMSMSMSMSMSMSMSMSMSMSMSMSMSMSMSMSMSMSMSMSMSMSMSMSMSMSMSMSMSMSMSMSMSMSMSMSMSMSMSMSMSMSMSMSMSMSMSMSMSMSMSMSMSMSMSMSMSMSMSMSMSMSMSMSMSMSMSMSMSMSMSMSMSMSMSMSMSMSMSMSMSMSMSMSMSMSMSMSMSMSMSMSMSMSMSMSMSMSMSMSMSMSMSMSMSMSMSMSMSMSMSMSMSMSMSMSMSMSMSMSMSMSMSMSMSMSMSMSMSMSMSMSMSMSMSMSMSMSMSMSMSMSMSMSMSMSMSMSMSMSMSMSMSMSMSMSMSMSMSMSMSMSMSMSMSMSMSMSMSMSMSMSMSMSMSMSMSMSMSMSMSMSMSMSMSMSMSMSMSMSMSMSMSMSMSMSMSMSMSMSMSMSMSMSMSMSMSMSMSMSMSMSMSMSMSMSMSMSMSMSMSMSMSMSMSMSMSMSMSMSMSMSMSMSMSMSMSMSMSMSMSMSMSMSMSMSMSMSMSMSMSMSMSMSMSMSMSMSMSMSMSMSMSMSMSMSMSMSMSMSMSMSMSMSMSMSMSMSMSMSMSMSMSMSMSMSMSMSMSMSMSMSMSMSMSMSMSMSMSMSMSMSMSMSMSMSMSMSMSMSMSMSMSMSMSMSMSMSMSMSMSMSMSMSMSMSMSMSMSMSMSMSMSMSMSMSMSMSMSMSMSMSMSMSMSMSMSMSMSMSMSMSMSMSMSMSMSMSMSMSMSMSMSMSMSMSMS MSMSMSMSMSMSMSMSMSMSMSMSMSMSMSMSMSMSMSMSMSMSMSMSMSMSMSMSMSMSMSMSMSMSMSMSMSMSMSMSMSMSMSMSMSMSMSMSMSMSMSMSMSMSMSMSMSMSMSMSMSMSMSMSMSMSMSMSMSMSMSMSMSMSMSMSMSMSMSMSMSMSMSMSMSMSMSMSMSMSMSMSMSMSMSMSMSMSMSMSMSMSMSMSMSMSMSMSMSMSMSMSMSMSMSMSMSMSMSMSMSMSMSMSMSMSMSMSMSMSMSMSMSMSMSMSMSMSMSMSMSMSMSMSMSMSMSMSMSMSMSMSMSMSMSMSMSMSMSMSMSMSMSMSMSMSMSMSMSMSMSMSMSMSMSMSMSMSMSMSMSMSMSMSMSMSMSMSMSMSMSMSMSMSMSMSMSMSMSMSMSMSMSMSMSMSMSMSMSMSMSMSMSMSMSMSMSMSMSMSMSMSMSMSMSMSMSMSMSMSMSMSMSMSMSMSMSMSMSMSMSMSMSMSMSMSMSMSMSMSMSMSMSMSMSMSMSMSMSMSMSMSMSMSMSMSMSMSMSMSMSMSMSMSMSMSMSMSMSMSMSMSMSMSMSMSMSMSMSMSMSMSMSMSMSMSMSMSMSMSMSMSMSMSMSMSMSMSMSMSMSMSMSMSMSMSMSMSMSMSMSMSMSMSMSMSMSMSMSMSMSMSMSMSMSMSMSMSMSMSMSMSMSMSMSMSMSMSMSMSMSMSMSMSMSMSMSMSMSMSMSMSMSMSMSMSMSMSMSMSMSMSMSMSMSMSMSMSMSMSMSMSMSMSMSMSMSMSMSMSMSMSMSMSMSMSMSMSMSMSMSMSMSMSMSMSMSMSMSMSMSMSMSMSMSMSMSMSMSMSMSMSMSMSMSMSMSMSMSMSMSMSMSMSMSMSMSMSMSMSMSMSMSMSMSMSMSMSMSMSMIMSMSMSMSMSMSMSMSMSMSMSMSMSMSMSMSMSMSMSMSMSMSMSMSMSMSMSMSMS} \\
\cmidrule(lr){1-2}
\centering\textbf{Evaluation} & 4 points: The sequence MS is repeated meaninglessly more than two times.\\
\bottomrule
\end{longtable}
\end{scriptsize}

\end{document}